\documentclass[11pt]{article}

\usepackage[a4paper,margin=1in]{geometry}
\usepackage{amsmath,amssymb,amsthm}
\usepackage{graphicx}
\usepackage{booktabs}
\usepackage{algorithm}
\usepackage{algpseudocode}
\usepackage[hidelinks]{hyperref}

\title{Signed p-adic Residual Encodings of Finite-Domain All-Different Systems\\with a Sudoku Case Study}
\author{Greg Baker\\\small School of Computing, Australian National University\\\small Canberra, Australia\\\small \href{mailto:greg.baker@anu.edu.au}{greg.baker@anu.edu.au}\\\small ORCID: \href{https://orcid.org/0000-0002-6910-3011}{0000-0002-6910-3011}}
\date{}
\hypersetup{
    pdftitle={Signed p-adic Residual Encodings of Finite-Domain All-Different Systems with a Sudoku Case Study},
    pdfauthor={Greg Baker},
    pdfkeywords={p-adic regression, signed residual objectives, all-different constraints, list colouring, Sudoku}
}

\newtheorem{theorem}{Theorem}
\newtheorem{corollary}[theorem]{Corollary}
\newtheorem{lemma}{Lemma}
\newtheorem{definition}{Definition}

\begin{document}
\maketitle

\begin{abstract}
We study signed, weighted affine \(p\)-adic residual objectives as native encodings of finite-domain constraints. For primes that separate the finite alphabet, sufficiently weighted positive unary rows pin each coefficient to its allowed set, while negative rows reward unequal endpoints or clause satisfaction. A coordinatewise domination theorem places every global minimiser in the finite domain; there the loss is, up to an additive constant, the all-different conflict count or the negative number of satisfied CNF clauses. Standard Sudoku provides an \(81\)-coefficient case study without a one-hot lift. A client-side implementation exposes the generated dataframes, arithmetic, diagnostics, and searches.
\end{abstract}


\begin{table}[t!]
\centering
\footnotesize
\begin{tabular}{@{}p{0.27\linewidth}p{0.54\linewidth}p{0.13\linewidth}@{}}
\toprule
Term used here & Classical reading & Reference \\
\midrule
finite-domain all-different system & Finite-domain constraint system in which the variables in each designated scope must take pairwise distinct values & Section~\ref{sec:introduction} \\
dataframe; row; synthetic observation & Tabular presentation of signed weighted affine forms; a row \((\sigma,u,b,\gamma)\) contributes \(\sigma\gamma|u^\top x-b|_p\). Positive rows are penalties and negative rows are bounded rewards & Definition~\ref{def:signed-dataset} \\
feature vector \(u\); target \(b\) & Normal vector and right-hand side defining the affine residual \(u^\top x-b\) & Definition~\ref{def:signed-dataset} \\
coefficient vector \(x\) & Optimisation variable in \(\mathbb Z_p^n\); in the Sudoku model its coordinates are cell values & Definition~\ref{def:signed-dataset} \\
compiler; compilation & Constructive encoding or reduction from a constraint system to a signed residual objective & Theorem~\ref{thm:compiler-template} \\
coordinatewise domination & Exact-penalty threshold: each pinning weight exceeds the total interaction weight incident on its coordinate & Theorem~\ref{thm:compiler-template} \\
pinning weight & Exact-penalty multiplier \(\lambda_i\) on the unary well of coordinate \(i\) & Theorem~\ref{thm:compiler-template} \\
list-colouring & Proper vertex colouring in which each vertex \(i\) draws its colour from its own list \(D_i\) & Definition~\ref{def:list-dataset} \\
well (unary, Boolean, digit) & Finite-alphabet multiwell penalty \(U_i(t)=\sum_{a\in D_i}|t-a|_p\), minimised exactly on the allowed set \(D_i\) & Lemma~\ref{lem:unary-well} \\
digit-snapping; snapping & Replacement of an out-of-domain coordinate by an allowed value; under the domination condition this strictly decreases the objective & Lemma~\ref{lem:digit-snapping} \\
primal graph; peer graph & Constraint (Gaifman) graph: variables are adjacent exactly when they share a constraint; the deduplicated Sudoku peer graph is \(20\)-regular, so its maximum degree is \(\Delta=20\) & Corollary~\ref{cor:all-different-csp} \\
Max-CSP; minimum conflicts & Maximisation of satisfied constraint weight, equivalently minimisation of conflict weight & Section~\ref{sec:minimum-conflicts-potts} \\
clue & Given digit in a pre-filled cell, represented by a singleton allowed set \(\{g_i\}\) & Section~\ref{sec:sudoku} \\
unit & Sudoku constraint scope: one row, one column, or one \(3\times3\) box; each induces a \(K_9\) clique in the peer graph & Section~\ref{sec:sudoku} \\
one-hot lift & Embedding of a \(q\)-valued variable as a \(0/1\) indicator vector of length \(q\) & Section~\ref{sec:parsimony} \\
local search; restart & Iterative exploration by local moves on a finite state space; a restart begins from a fresh random initialisation & Section~\ref{sec:illustrative-computations} \\
QUBO & Quadratic unconstrained binary optimisation: minimisation of a quadratic form over \(\{0,1\}^n\) & Section~\ref{sec:related-work} \\
\parbox[t]{\linewidth}{\raggedright positive complement\\expansion} & Rewriting of a negative reward row as positive wells at the allowed residual values; equal up to an additive constant on the finite domain & Appendix~\ref{app:mihara-digitwise} \\
carving & Generation of a puzzle by deleting entries from a solved grid; the random procedure used here does not enforce a unique completion & Appendix~\ref{app:heuristics-results} \\
\bottomrule
\end{tabular}
\caption{Correspondence between the applied vocabulary used in this paper and standard mathematical or constraint-programming readings.}
\label{tab:terminology}
\end{table}

\section{Introduction}
\label{sec:introduction}

In affine \(p\)-adic linear regression, an observation pairs a feature vector \(u\in\mathbb Z^n\) with a target \(b\in\mathbb Z\), and a coefficient vector \(x\) is scored by a weighted sum of residual norms \(|u^\top x-b|_p\). It is analogous to Euclidean linear regression, but uses \(p\)-adic residual norms in place of squared Euclidean residuals.

Earlier work establishes the geometry and complexity of this problem when every weight is positive \cite{baker2025padic,baker2026nphard,mihara2026forcing}. This paper studies the signed variant, in which an observation may carry a negative weight. In ordinary (Euclidean) least squares, negative weights can destroy convexity and boundedness (as discussed in Section~\ref{sec:archimedean-comparison}). Over \(\mathbb Z_p\), by contrast, every affine residual has norm at most one, so a finite signed objective remains bounded and a negative row acts as a bounded reward.

This makes \(p\)-adic linear regression surprisingly expressive. Finite-domain all-different systems and CNF formulae compile mechanically into signed objectives: sufficiently weighted positive unary rows pin each coefficient to a finite allowed set, and a coordinatewise domination theorem (Theorem~\ref{thm:compiler-template}) shows that every global minimiser lies in that finite domain. There each remaining residual norm is a zero-or-one indicator, so a negative row rewards an unequal pair \(x_i\neq x_j\) or a satisfied clause, and the loss is, up to an additive constant, the conflict count of the all-different system or the negative number of satisfied clauses. The global minimisers are exactly the domain-respecting minimum-conflict assignments, and exactly the satisfying assignments when the instance is satisfiable (Theorem~\ref{thm:all-different}, Corollary~\ref{cor:all-different-csp}).

This expressivity also gives a short proof that signed \(p\)-adic linear regression is NP-hard: the clause compiler reduces \textsc{3-SAT} to the threshold problem at the fixed prime \(p=5\) (Corollary~\ref{cor:signed-nphard}). Earlier NP-hardness results treated an ordinary summed \(2\)-adic objective \cite{baker2026nphard}, while Mihara's forcing theorem proves the positive-weight problem NP-hard for every fixed prime \cite{mihara2026forcing}.

The paper makes the following contributions.
\begin{itemize}
\item A compiler template from finite-domain all-different systems and CNF formulae to signed weighted affine \(p\)-adic residual objectives, with a coordinatewise domination theorem locating every global minimiser in the finite domain, on \(\mathbb Z_p^n\) and on \(\mathbb Q_p^n\) (Theorem~\ref{thm:compiler-template}, Corollary~\ref{cor:qp-extension}).
\item An exact characterisation of the global minimisers for all-different systems in list-colouring form as the domain-respecting minimum-conflict assignments, covering the unsatisfiable Max-CSP regime and the antiferromagnetic Potts reading (Theorem~\ref{thm:all-different}, Section~\ref{sec:minimum-conflicts-potts}).
\item NP-hardness of the signed threshold problem at the fixed prime \(p=5\), by an elementary reduction through the clause compiler (Corollary~\ref{cor:signed-nphard}).
\item Polynomial-size positive-only complement expansions for the negative CNF and pairwise all-different rows, supporting a diagnostic comparison with Mihara's digitwise equality regression while the signed rows are retained for valuation and loss reporting (Appendix~\ref{app:mihara-digitwise}).
\item A Sudoku case study on the native \(81\)-cell digit space using a variety of \(p\)-adic linear-regression techniques; a negative result for a powers-of-two encoding; and a client-side browser implementation exposing the generated dataframes, loss arithmetic, and searches (Sections~\ref{sec:sudoku}--\ref{sec:illustrative-computations}, Appendices~\ref{app:heuristics}--\ref{app:powers-two}).
\end{itemize}

This interdisciplinary work is written from the perspective of computer science and applied mathematics.
The construction is described throughout in the vocabulary of constraint satisfaction and machine learning: a constraint system is \emph{compiled} into a \emph{dataframe} of observations, positive \emph{wells} \emph{pin} coefficients to their finite domains, and global minimisers \emph{snap} onto those domains. Table~\ref{tab:terminology} maps this vocabulary to standard language from penalty-function theory, graph theory, and constraint programming.

Section~\ref{sec:preliminaries} describes the \(p\)-adic notation and the signed-observation format.
Section~\ref{sec:padic-regression-background} provides a worked three-variable list-colouring example and a two-variable CNF example in full, and records a family of integer false labellings with the uniform labelling as its degenerate member and a member whose residual values identify the satisfied literals.

Section~\ref{construction} states and proves the general template and its corollaries. Theorem~\ref{thm:compiler-template} is an exact-penalty theorem in the ultrametric setting: the strong triangle inequality puts a bound on the interaction perturbation of any coordinate change, so pinning weights that are above the weight of the incident interactions ensure that every minimiser has its parameters in the desired set of distinct values. The section closes by comparing the signed objective with its Archimedean analogue, for which the corresponding move fails twice over: a sufficiently weighted negative row destroys boundedness, and convex penalties cannot pin a real variable to a finite non-singleton domain.

Section~\ref{sec:minimum-conflicts-potts} considers unsatisfiable instances. Sections~\ref{sec:sudoku} and~\ref{sec:parsimony} specialise the construction to Sudoku. Section~\ref{sec:illustrative-computations} reports the computations. Section~\ref{sec:related-work} discusses related work. The appendices compare Mihara's digitwise regression, record the heuristics and their traces, and report the experiments and the failed powers-of-two encoding.

\section{Preliminaries}
\label{sec:preliminaries}

Fix a prime \(p\). For a nonzero integer \(n\), let \(v_p(n)\) be the exponent of \(p\) in its factorisation and define
\[
|n|_p=p^{-v_p(n)},
\qquad
|0|_p=0.
\]
Setting \(v_p(a/b)=v_p(a)-v_p(b)\) extends \(|\cdot|_p\) multiplicatively to \(\mathbb Q\). The field \(\mathbb Q_p\) of \(p\)-adic numbers is the completion of \(\mathbb Q\) with respect to \(|\cdot|_p\), and the ring of \(p\)-adic integers is
\[
\mathbb Z_p=\{t\in\mathbb Q_p:|t|_p\le 1\}.
\]
This norm makes numbers close when their difference is divisible by a large power of \(p\). It is non-Archimedean: beyond the ordinary triangle inequality it satisfies the strong triangle inequality
\[
|s+t|_p\le\max\bigl(|s|_p,|t|_p\bigr),
\]
so \(d(s,t):=|s-t|_p\) is an ultrametric on \(\mathbb Q_p\).

Every integer lies in \(\mathbb Z_p\), so every residual formed below from integer data has norm at most \(1\).

If \(s\neq t\) are integers with \(1\le s,t\le q<p\), then \(p\nmid s-t\) and hence \(|s-t|_p=1\): on a finite alphabet separated by \(p\), the residual norm is exactly a zero-or-one inequality indicator.

In affine \(p\)-adic linear regression, an observation with feature vector \(u\in\mathbb Z^n\) and target \(b\in\mathbb Z\) has residual \(u^\top x-b\) at the coefficient vector \(x\in\mathbb Z_p^n\), and the usual positive-weight objective sums terms of the form
\[
\gamma\,|u^\top x-b|_p,
\qquad \gamma>0.
\]
This is a weighted \(p\)-adic linear-regression objective.
Earlier results establish geometric and complexity properties of this positive-weight problem \cite{baker2025padic,baker2026nphard}; broader recent work develops classification, regression, and representation-learning primitives over the \(p\)-adics \cite{martins2025learning}. We find no prior treatment of the signed affine objective.

\begin{definition}[Signed synthetic $p$-adic residual dataset]
\label{def:signed-dataset}
For parameters \(x\in \mathbb Z_p^n\), a signed weighted affine observation is a quadruple \((\sigma,u,b,\gamma)\), where \(\sigma\in\{+1,-1\}\) is the sign, \(u\in\mathbb Z^n\) is the feature vector, \(b\in\mathbb Z\) is the target, and \(\gamma>0\) is the weight. It contributes the term
\[
\sigma\gamma\,|u^\top x-b|_p
\]
to the loss; this is the basic term in the signed weighted $p$-adic objective.
\end{definition}

Since \(u\in\mathbb Z^n\), \(b\in\mathbb Z\), and \(x\in\mathbb Z_p^n\), each residual lies in \(\mathbb Z_p\) and has norm at most \(1\). Every finite signed objective considered on this domain is therefore bounded. An observation is \emph{positive} when \(\sigma=+1\) and \emph{negative} when \(\sigma=-1\); positive observations will pin coefficients and negative observations will supply bounded rewards.

Over \(\mathbb R\) we can't do this for two independent reasons: a sufficiently weighted negative row destroys boundedness, and convex penalties cannot pin a variable to a finite non-singleton set.

\section{Worked examples}
\label{sec:padic-regression-background}

Signed residual objectives give useful representations of constraint-satisfaction problems: source forms can be transformed mechanically into datasets, and a global minimiser can be transformed mechanically into a satisfying CSP assignment when one exists, or otherwise into a domain-respecting minimum-conflict assignment. This section demonstrates the construction on two small instances before the general theorems of Section~\ref{construction}.

\begin{figure}[H]
\centering
\includegraphics[width=7cm]{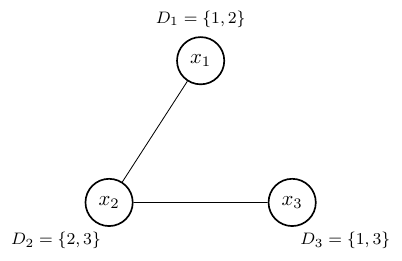}
\caption{A tiny list-colouring instance. The assignment \(x_1=1\), \(x_2=2\), \(x_3=3\) respects the domains and makes every edge unequal.}
\label{fig:list-colouring-example}
\end{figure}

\begin{table}[t]
\centering
\small
\begin{tabular}{@{}crrrrc l@{}}
\toprule
Row & $x_1$ & $x_2$ & $x_3$ & Target & Signed weight & Role \\
\midrule
$P_1$ & $1$ & $0$ & $0$ & $1$ & $+3$ & domain well \\
$P_2$ & $1$ & $0$ & $0$ & $2$ & $+3$ & domain well \\
$P_3$ & $0$ & $1$ & $0$ & $2$ & $+3$ & domain well \\
$P_4$ & $0$ & $1$ & $0$ & $3$ & $+3$ & domain well \\
$P_5$ & $0$ & $0$ & $1$ & $1$ & $+3$ & domain well \\
$P_6$ & $0$ & $0$ & $1$ & $3$ & $+3$ & domain well \\
$E_1$ & $1$ & $-1$ & $0$ & $0$ & $-1$ & reward $x_1\ne x_2$ \\
$E_2$ & $0$ & $1$ & $-1$ & $0$ & $-1$ & reward $x_2\ne x_3$ \\
\bottomrule
\end{tabular}
\caption{Dataframe view of Figure~\ref{fig:list-colouring-regression}. The three coefficient columns form the feature vector, and each row contributes its signed weight times the \(5\)-adic norm of feature dot coefficient minus target.}
\label{tab:list-colouring-dataframe}
\end{table}

\begin{figure}[t]
\centering
\small
\begin{tabular}{@{}cllp{0.27\linewidth}@{}}
\toprule
Row & Type & Synthetic observation & Role \\
\midrule
$P_1$ & positive & $(+1,e_1,1,3)$ & pin $x_1$ to $D_1$ \\
$P_2$ & positive & $(+1,e_1,2,3)$ & pin $x_1$ to $D_1$ \\
$P_3$ & positive & $(+1,e_2,2,3)$ & pin $x_2$ to $D_2$ \\
$P_4$ & positive & $(+1,e_2,3,3)$ & pin $x_2$ to $D_2$ \\
$P_5$ & positive & $(+1,e_3,1,3)$ & pin $x_3$ to $D_3$ \\
$P_6$ & positive & $(+1,e_3,3,3)$ & pin $x_3$ to $D_3$ \\
$E_1$ & negative & $(-1,e_1-e_2,0,1)$ & reward $x_1\neq x_2$ \\
$E_2$ & negative & $(-1,e_2-e_3,0,1)$ & reward $x_2\neq x_3$ \\
\bottomrule
\end{tabular}
\caption{Synthetic dataset for the toy list-colouring instance in Figure~\ref{fig:list-colouring-example}. Here \(e_1,e_2,e_3\) denote the standard basis vectors of \(\mathbb Z^3\).}
\label{fig:list-colouring-regression}
\end{figure}

\subsection{An example}
\label{sec:worked-all-different}

Figure~\ref{fig:list-colouring-example} has three variables. Each variable must take a value from its domain, and adjacent variables must be different. The assignment \(x=(1,2,3)\) is one solution.

For this example take \(p=5\). The only fact about the \(5\)-adic norm needed at first is that \(|0|_5=0\), and that every nonzero integer residual that occurs has norm \(1\). A dataframe row consists of a feature vector \(u\), target \(b\), and signed weight \(w\); at coefficient vector \(x\) it contributes
\[
w\,|u^\top x-b|_5
\]
to the objective. Positive rows of weight \(3\) pin each coefficient, while rows of weight \(-1\) reward unequal endpoints.

Table~\ref{tab:list-colouring-dataframe} and Figure~\ref{fig:list-colouring-regression} show the eight observations; Figure~\ref{fig:list-colouring-regression} uses the compact tuple notation of Definition~\ref{def:signed-dataset}, in which each tuple \((\sigma,u,b,\gamma)\) has signed weight \(w=\sigma\gamma\).

The zero-intercept model is the hyperplane
\[
b=u^\top x=u_1x_1+u_2x_2+u_3x_3.
\]
The general theorem below (Theorem~\ref{thm:compiler-template} on page~\pageref{thm:compiler-template}) guarantees that at least one global minimum occurs among the finitely many domain-respecting candidates.

On a domain-respecting vector the six positive rows contribute \(9\), while each unequal edge subtracts one. Thus
\[
L_{\mathrm{toy}}(x)=9-\mathbf 1_{x_1\ne x_2}-\mathbf 1_{x_2\ne x_3}.
\]
Brute-force evaluation gives Table~\ref{tab:list-colouring-candidates}.

\begin{table}[t]
\centering
\small
\begin{tabular}{@{}ccc@{}}
\toprule
Coefficient vector $x$ & $L_{\mathrm{toy}}(x)$ & Proper list-colouring? \\
\midrule
$(1,2,1)$ & $7$ & yes \\
$(1,2,3)$ & $7$ & yes \\
$(1,3,1)$ & $7$ & yes \\
$(1,3,3)$ & $8$ & no \\
$(2,2,1)$ & $8$ & no \\
$(2,2,3)$ & $8$ & no \\
$(2,3,1)$ & $7$ & yes \\
$(2,3,3)$ & $8$ & no \\
\bottomrule
\end{tabular}
\caption{All eight hyperplanes determined by the domain-pinning rows.}
\label{tab:list-colouring-candidates}
\end{table}

The minimum-loss coefficient vectors are the valid assignments of Figure~\ref{fig:list-colouring-example}; in particular, \(x=(1,2,3)\) is recovered as a regression optimum.

\subsection{\texorpdfstring{\(p\)-adic}{p-adic} linear regression can express CNF statements as well}
\label{sec:worked-cnf}

Consider the CNF expression in the two Boolean variables \(z_1\) and \(z_2\):
\[
\Phi=(z_1)\wedge(\neg z_1\vee z_2).
\]
What values of \(z_1\) and \(z_2\) make \(\Phi\) true? Encode \(x_i=0\) as \(z_i=\mathrm{true}\) and \(x_i=1\) as \(z_i=\mathrm{false}\). A clause is false at one affine equality, so a negative row rewards every coefficient vector away from that equality. The first clause is false when \(x_1=1\), giving the residual \(x_1-1\). The second is false when \((x_1,x_2)=(0,1)\), giving \(-x_1+x_2-1\). With \(p=5\) and Boolean-pinning weight \(3\), the complete dataframe is Table~\ref{tab:cnf-dataframe}.

We want to make sure that \(x_1\) and \(x_2\) take values only in \(\{0,1\}\), so we make two deep Boolean wells in the loss landscape for each variable. If the wells are deeper than the rewards from the clause-derived rows, then the optimal values of \(x_1\) and \(x_2\) are suitably constrained.

\begin{table}[H]
\centering
\small
\begin{tabular}{@{}crrrcl@{}}
\toprule
Row & $x_1$ & $x_2$ & Target & Signed weight & Role \\
\midrule
$C_1$ & $1$ & $0$ & $1$ & $-1$ & reward $(z_1)$ \\
$C_2$ & $-1$ & $1$ & $1$ & $-1$ & reward $(\neg z_1\vee z_2)$ \\
$U_{1,0}$ & $1$ & $0$ & $0$ & $+3$ & Boolean well \\
$U_{1,1}$ & $1$ & $0$ & $1$ & $+3$ & Boolean well \\
$U_{2,0}$ & $0$ & $1$ & $0$ & $+3$ & Boolean well \\
$U_{2,1}$ & $0$ & $1$ & $1$ & $+3$ & Boolean well \\
\bottomrule
\end{tabular}
\caption{Signed regression dataframe compiled from the two-clause formula \(\Phi\). Constraint rows reward nonzero residuals; well rows force Boolean coefficients.}
\label{tab:cnf-dataframe}
\end{table}

Choosing one well for \(x_1\) and one for \(x_2\) enumerates the four zero-intercept hyperplanes through the corresponding two basis rows. On Boolean coefficients the wells contribute \(6\), so the loss is \(6-\operatorname{sat}_{\Phi}(x)\):

\begin{table}[H]
\centering
\small
\begin{tabular}{@{}cccc@{}}
\toprule
$x$ & $(z_1,z_2)$ & Clauses satisfied & Loss \\
\midrule
$(0,0)$ & $(\mathrm{true},\mathrm{true})$ & $2$ & $4$ \\
$(0,1)$ & $(\mathrm{true},\mathrm{false})$ & $1$ & $5$ \\
$(1,0)$ & $(\mathrm{false},\mathrm{true})$ & $1$ & $5$ \\
$(1,1)$ & $(\mathrm{false},\mathrm{false})$ & $1$ & $5$ \\
\bottomrule
\end{tabular}
\caption{Brute-force regression for \(\Phi\).}
\label{tab:cnf-candidates}
\end{table}

The unique minimum has coefficients \(x=(0,0)\), which decode to the satisfying assignment \(z_1=z_2=\mathrm{true}\).

\subsection{General false labels}
\label{sec:general-labels}

The preceding CNF uses \(0\) for true and \(1\) for false for both variables. The false values need not coincide. They could be other \(p\)-adic integers, and they need not be distinct.

Within \(\mathbb Z_p\) itself, the signed clause construction works verbatim for any false labels \(b_i\equiv 1\pmod p\). Because integer norms are at most \(1\), the domination arguments of Section~\ref{construction} are unaffected. Distinct labels add information: with \(b_i=1+p\,2^{i}\), a subset of \(k\) labels sums to \(k+p\sum_{i}2^{i}\), so distinct subsets have distinct sums---a cardinality collision would require \(p\) to divide a nonzero difference of cardinalities smaller than \(p\)---and the residual value then identifies exactly which literals of a clause are satisfied: the count from the residue modulo \(p\), the set from the binary expansion of the quotient. The companion browser's \href{https://padic-logic.symmachus.org/\#csp}{Boolean-CSP page} implements both labellings: switching to \(b_i=1+p\,2^{i}\) leaves every loss unchanged, while its row-evaluation table decodes each clause residual into the satisfied literals it names. The general conditions under which labels and clause coefficients preserve the clause-count loss, an encoding of negation by non-negative coefficients, and a polynomial-valued variant of these labels are developed in the author's doctoral thesis.

\section{A general all-different residual-objective template}
\label{construction}

The worked examples can now be put on a general footing.
We work in the $p$-adic integers $\mathbb Z_p$ of Section~\ref{sec:preliminaries}, where every norm is at most $1$. If $p>q$, the difference between any two distinct values in $\{1,\dots,q\}$ has norm $1$. For comparison, the discrete metric \(d_{\mathrm{disc}}(s,t):=\mathbf 1_{s\neq t}\) is also an ultrametric; a discrete version of the construction is stated after the main theorem.

Many finite-domain all-different systems can be viewed as list-colouring problems. Each variable $x_i$ has a finite allowed set $D_i$, and an edge \(\{i,j\}\) requires \(x_i\neq x_j\). An all-different constraint of arity \(r\) contributes the \(\binom{r}{2}\) pairwise edges of its primal-graph clique. The observations below are the signed weighted affine observations of Definition~\ref{def:signed-dataset}.

\begin{theorem}[Finite-domain signed affine compiler template]
\label{thm:compiler-template}
Fix a prime \(p\). Let \(D_i\subseteq\mathbb Z\) be finite nonempty sets whose elements are pairwise incongruent modulo \(p\), and put
\[
D:=\prod_{i=1}^n D_i\subseteq\mathbb Z_p^n.
\]
For each coordinate \(i\), choose a pinning weight \(\lambda_i>0\). Let \(\mathcal T\) be a finite family of signed affine observations
\[
(\sigma_\ell,u_\ell,b_\ell,\gamma_\ell),
\qquad
\sigma_\ell\in\{\pm1\},\quad u_\ell\in\mathbb Z^n,\quad b_\ell\in\mathbb Z,\quad \gamma_\ell>0,
\]
and define
\[
L(x):=
\sum_{i=1}^n\lambda_i\sum_{a\in D_i}|x_i-a|_p
\;+\;
\sum_{\ell\in\mathcal T}\sigma_\ell\gamma_\ell |u_\ell^\top x-b_\ell|_p.
\]
Assume that, for every coordinate \(i\),
\[
\lambda_i>\sum_{\ell:\,(u_\ell)_i\neq 0}\gamma_\ell,
\]
and that, for every \(x\in D\) and every \(\ell\in\mathcal T\), the residual \(u_\ell^\top x-b_\ell\) is either \(0\) or a \(p\)-adic unit. Then \(L\) attains its global minimum on \(D\), and every global minimiser lies in \(D\). On \(D\), \(L\) is a constant plus a finite signed objective whose interaction terms have values \(0\) and \(\sigma_\ell\gamma_\ell\).
\end{theorem}

\begin{proof}
Suppose \(x_i\notin D_i\). If \(x_i\) is congruent modulo \(p\) to an element of \(D_i\), let \(b\) be that element; otherwise choose any \(b\in D_i\). Let \(x'\) be obtained from \(x\) by replacing \(x_i\) with \(b\).

The \(i\)-th unary sum decreases by at least \(\lambda_i|x_i-b|_p\). Indeed, if \(x_i\equiv b\pmod p\), then the only changed unary contribution below \(1\) is \(|x_i-b|_p\). If \(x_i\) is congruent to no element of \(D_i\), then \(|x_i-a|_p=1\) for all \(a\in D_i\), while the unary sum at \(b\) has one zero term and all other terms equal to \(1\); in this case \(|x_i-b|_p=1\).

For an interaction term with \((u_\ell)_i\neq0\), write \(r_\ell'=u_\ell^\top x'-b_\ell\). Then
\[
u_\ell^\top x-b_\ell=r_\ell'+(u_\ell)_i(x_i-b).
\]
Since \((u_\ell)_i\in\mathbb Z\), \(|(u_\ell)_i|_p\le1\), and the ultrametric inequality gives
\[
\bigl||u_\ell^\top x-b_\ell|_p-|r_\ell'|_p\bigr|
\le |x_i-b|_p.
\]
Thus the total interaction increase caused by changing coordinate \(i\) is at most
\[
\sum_{\ell:\,(u_\ell)_i\neq0}\gamma_\ell |x_i-b|_p.
\]
By the assumed domination inequality, \(L(x')<L(x)\). Repeating this coordinate-snapping step gives a point of \(D\) with smaller loss than any starting point outside \(D\). Since \(D\) is finite, the global minimum is attained on \(D\), and no global minimiser lies outside \(D\).

On \(D\), every unary sum is constant: exactly one term is zero and the rest are units. The second assumption says each interaction residual is either zero or a unit, so each signed interaction term is either \(0\) or \(\sigma_\ell\gamma_\ell\). This proves the stated finite-domain form of \(L|_D\).
\end{proof}

\begin{corollary}[Extension to \(\mathbb Q_p^n\)]
\label{cor:qp-extension}
Under the assumptions of Theorem~\ref{thm:compiler-template}, view the same formula for \(L\) as an objective on \(\mathbb Q_p^n\). It remains bounded below, attains its global minimum on \(D\), and has no global minimiser outside \(D\).
\end{corollary}

\begin{proof}
Theorem~\ref{thm:compiler-template} handles every coordinate in \(\mathbb Z_p\setminus D_i\). If instead \(|x_i|_p>1\), then \(|x_i-a|_p=|x_i|_p\) for every integer \(a\in D_i\). Replacing \(x_i\) by any \(b\in D_i\) decreases its unary sum by
\[
\lambda_i\bigl(|D_i|\,|x_i|_p-(|D_i|-1)\bigr)
\geq \lambda_i|x_i|_p
=\lambda_i|x_i-b|_p.
\]
The interaction increase is bounded by \(\sum_{\ell:(u_\ell)_i\neq0}\gamma_\ell|x_i-b|_p\), exactly as in the theorem. The domination inequality therefore makes the replacement strictly improving. Snapping all coordinates first into \(\mathbb Z_p\) and then into \(D\) proves the claim. In particular, every value of \(L\) is bounded below by the finite minimum attained on \(D\).
\end{proof}

\begin{definition}[Synthetic dataset for list-colouring]
\label{def:list-dataset}
Let \(V=\{1,\dots,n\}\), let \(G=(V,E)\) be a finite graph, let \(D_i\subseteq \{1,\dots,q\}\) be nonempty allowed sets, and fix a prime \(p>q\). Define
\[
w_{ii}:=0,
\qquad
w_{ij}:=
\begin{cases}
1, & \{i,j\}\in E,\\
0, & \{i,j\}\notin E,
\end{cases}
\qquad (i\neq j),
\]
let \(\Delta:=\max_{i\in V}\sum_{j=1}^n w_{ij}\), and set \(\lambda:=1+\Delta\). Writing \(e_1,\dots,e_n\) for the standard basis vectors of \(\mathbb Z^n\), define
\[
\mathcal S^+ := \{(+1,e_i,a,\lambda): i\in V,\ a\in D_i\},
\]
\[
\mathcal S^- := \{(-1,e_i-e_j,0,1): 1\le i<j\le n,\ w_{ij}=1\}.
\]
Set \(\mathcal S:=\mathcal S^+\sqcup\mathcal S^-\).
\end{definition}

The theorem is stated for a simple graph, so \(w_{ij}\in\{0,1\}\). An edge-weighted version is obtained by replacing the unit negative observation on edge \(\{i,j\}\) with weight \(c_{ij}>0\) and choosing the pinning weight larger than \(\max_i\sum_j c_{ij}\), but the binary case is the one needed here.

\begin{definition}[Loss associated with the list-colouring dataset]
\label{def:list-loss}
For \(\mathcal S\) as in Definition~\ref{def:list-dataset}, define
\[
L_{\mathcal S}(x)
=
\lambda\sum_{i=1}^n \sum_{a\in D_i}|x_i-a|_p
-\sum_{1\le i<j\le n} w_{ij}|x_i-x_j|_p.
\]
\end{definition}

This dataset has exactly \(\sum_{i=1}^n |D_i|\) positive observations and \(|E|\) negative observations, so it is computable in time \(O(\sum_i |D_i|+|E|)\). The positive observations pin coordinates to allowed values, and the negative observations reward inequality across edges.

Each positive observation \((+1,e_i,a,\lambda)\) contributes \(\lambda|x_i-a|_p\) to the loss, and each negative observation \((-1,e_i-e_j,0,1)\) contributes \(-|x_i-x_j|_p\) to the loss. Summing them yields the displayed loss. The size bound and runtime are immediate from the explicit definitions.

The compiler template already implies that minimisers snap to the finite alphabet, because the pinning weight is larger than the degree of every vertex. The next proof records the stronger fact needed for all-different systems: after snapping, the objective is exactly a constant plus the edge-conflict count.

\subsection{Correctness}
Since \(p>q\), the residues of \(1,\dots,q\) are distinct modulo \(p\). In particular, for any fixed \(i\in V\) and any \(t\in\mathbb Z_p\), there is at most one \(a\in D_i\) with \(t\equiv a\pmod p\).

First we record the simple observation that, on \(\{1,\dots,q\}\), the \(p\)-adic norm detects inequality exactly.

\begin{lemma}[Edge terms become inequality indicators on the domain]
\label{lem:edge-indicator}
If \(x_i,x_j\in \{1,\dots,q\}\) and \(p>q\), then
\[
|x_i-x_j|_p=
\begin{cases}
0,&x_i=x_j,\\
1,&x_i\neq x_j.
\end{cases}
\]
\end{lemma}

\begin{proof}
If \(x_i=x_j\) then the difference is zero. If \(x_i\neq x_j\), then \(1\le |x_i-x_j|\le q-1<p\), so \(p\nmid (x_i-x_j)\) and therefore \(|x_i-x_j|_p=1\).
\end{proof}

We next prove a digit-snapping lemma: if \(x\) minimises \(L_{\mathcal S}\), then each coordinate \(x_i\) lies in \(D_i\). The explicit unary-well formula makes that argument transparent.

\begin{lemma}[Unary wells]
\label{lem:unary-well}
For \(i\in V\), define
\[
U_i(t):=\sum_{a\in D_i}|t-a|_p
\qquad (t\in\mathbb Z_p).
\]
Then \(U_i(t)\ge |D_i|-1\), with equality if and only if \(t\in D_i\).
\end{lemma}

\begin{proof}
If \(t\in D_i\), then one summand in \(U_i(t)\) is \(0\), and the remaining \(|D_i|-1\) summands are \(1\), so \(U_i(t)=|D_i|-1\).

If \(t\notin D_i\) but \(t\equiv b \pmod p\) for some \(b\in D_i\), then \(0<|t-b|_p<1\), while for every \(a\in D_i\setminus\{b\}\) we have \(t-a\equiv b-a\not\equiv 0\pmod p\), so \(|t-a|_p=1\). Hence
\[
U_i(t)=(|D_i|-1)+|t-b|_p > |D_i|-1.
\]

If \(t\) is congruent modulo \(p\) to no element of \(D_i\), then \(|t-a|_p=1\) for all \(a\in D_i\), so \(U_i(t)=|D_i|>|D_i|-1\).
\end{proof}

\begin{lemma}[Digit-snapping lemma]
\label{lem:digit-snapping}
Let \(x\in\mathbb Z_p^n\). Suppose there exists \(i\in V\) with \(x_i\notin D_i\). Choose \(b\in D_i\) as follows: if \(x_i\) is congruent modulo \(p\) to an element of \(D_i\), take that unique element; otherwise choose any \(b\in D_i\). Let \(x'\) be obtained from \(x\) by replacing \(x_i\) with \(b\). Then
\[
L_{\mathcal S}(x')<L_{\mathcal S}(x).
\]
\end{lemma}

\begin{proof}
The positive part of the loss changes by
\[
\lambda\left(\sum_{a\in D_i}|b-a|_p-\sum_{a\in D_i}|x_i-a|_p\right).
\]
There are two cases to consider.

If \(x_i\) is congruent modulo \(p\) to no element of \(D_i\), then \(|x_i-a|_p=1\) for every \(a\in D_i\), so \(\sum_{a\in D_i}|x_i-a|_p=|D_i|\). On the other hand, one summand in \(\sum_{a\in D_i}|b-a|_p\) is \(0\) and the remaining \(|D_i|-1\) summands are \(1\), so \(\sum_{a\in D_i}|b-a|_p=|D_i|-1\). Hence the positive change is \(-\lambda\). Since \(|x_i-b|_p=1\) in this case, this is at most \(-\lambda |x_i-b|_p\).

If \(x_i\equiv b \pmod p\) for some \(b\in D_i\), then \(\sum_{a\in D_i}|b-a|_p=|D_i|-1\), whereas
\[
\sum_{a\in D_i}|x_i-a|_p=(|D_i|-1)+|x_i-b|_p.
\]
Hence the positive change is \(-\lambda |x_i-b|_p\).

For each \(j\in V\), the corresponding negative component changes by
\[
w_{ij}\bigl(|x_i-x_j|_p-|b-x_j|_p\bigr).
\]
By the ultrametric inequality,
\[
|x_i-x_j|_p \le \max\bigl(|x_i-b|_p,\ |b-x_j|_p\bigr),
\]
so
\[
|x_i-x_j|_p-|b-x_j|_p \le |x_i-b|_p.
\]
Summing over all \(j\in V\),
\[
L_{\mathcal S}(x')-L_{\mathcal S}(x)
\le
\left(-\lambda+\sum_{j=1}^n w_{ij}\right)|x_i-b|_p
\le
-|x_i-b|_p
<
0.
\]
\end{proof}

On domain-respecting assignments, the loss is an additive constant plus the conflict count.

\begin{theorem}[Correctness of the synthetic reduction]
\label{thm:all-different}
For the dataset \(\mathcal S=\mathcal S^+\sqcup\mathcal S^-\) from Definition~\ref{def:list-dataset} and the loss from Definition~\ref{def:list-loss}, every global minimiser of \(L_{\mathcal S}\) lies in \(\prod_{i=1}^n D_i\). On \(\prod_{i=1}^n D_i\),
\[
L_{\mathcal S}(x)
=
\lambda\sum_{i=1}^n (|D_i|-1)
-|E|
+\sum_{1\le i<j\le n} w_{ij}\mathbf 1_{x_i=x_j}.
\]
Consequently, the global minimisers are exactly the domain-respecting assignments that minimise the conflict sum \(\sum_{1\le i<j\le n} w_{ij}\mathbf 1_{x_i=x_j}\), or equivalently maximise the number of unequal edges. In particular, if the list-colouring instance \((G,(D_i)_{i=1}^n)\) is satisfiable, then this minimum conflict sum is \(0\), and the global minimisers are exactly its proper list-colourings.
\end{theorem}

\begin{proof}
By Lemma~\ref{lem:digit-snapping}, any point with some coordinate outside its domain can be improved by snapping that coordinate into the domain. Therefore every global minimiser lies in \(\prod_i D_i\). If \(x\in \prod_i D_i\), then Lemma~\ref{lem:unary-well} gives \(U_i(x_i)=|D_i|-1\) for every \(i\), and Lemma~\ref{lem:edge-indicator} gives
\[
|x_i-x_j|_p=\mathbf 1_{x_i\neq x_j}
\qquad (1\le i<j\le n).
\]
Since \(\mathbf 1_{x_i\neq x_j}=1-\mathbf 1_{x_i=x_j}\) on every edge, substituting these identities into \(L_{\mathcal S}\) yields the displayed formula. A global minimiser exists: \(L_{\mathcal S}\) takes only finitely many values on the finite set \(\prod_i D_i\), and by the previous paragraph no point outside \(\prod_i D_i\) can be a global minimiser, so the global minimum is attained on \(\prod_i D_i\). There the displayed formula is an additive constant plus the conflict sum, so a domain-respecting assignment is a global minimiser if and only if it attains the least possible conflict sum; this is the stated ``exactly'' characterisation in both directions. In particular, if \((G,(D_i)_{i=1}^n)\) is satisfiable the least conflict sum is \(0\), and the global minimisers are exactly its proper list-colourings.
\end{proof}

\begin{corollary}
\label{cor:all-different-csp}
Any finite-domain all-different constraint system can be encoded, in polynomial time, as a signed weighted synthetic $p$-adic residual objective whose global minimisers are exactly the domain-respecting assignments with minimum conflict count. If the constraint system is satisfiable, these minimisers are exactly its satisfying assignments.
\end{corollary}

\begin{proof}
Relabel the union of the explicitly listed domain values by \(\{1,\dots,q\}\), preserving equality across domains. Here \(q\leq\sum_i|D_i|\). Choose a prime \(p>q\) (for \(q>1\), Bertrand's postulate gives one below \(2q\)); because \(q\) is bounded by the explicit input length, scanning this interval with deterministic primality testing is polynomial in that input length. Form the primal graph: the vertices are the variables, and two variables are adjacent if they appear together in some all-different constraint. The satisfying assignments are exactly the proper list-colourings of this graph, with the variable domains as lists, and the minimum-conflict assignments are exactly the domain-respecting assignments minimising the corresponding deduplicated edge-conflict sum. Apply Definitions~\ref{def:list-dataset} and \ref{def:list-loss}, and Theorem~\ref{thm:all-different}.
\end{proof}

\paragraph{Discrete-metric variant.}
The same synthetic dataset also works with the discrete metric
\[
d_{\mathrm{disc}}(s,t):=\mathbf 1_{s\neq t}.
\]
If each residual \(|u^\top x-b|_p\) is replaced by \(d_{\mathrm{disc}}(u^\top x,b)\), then the loss becomes
\[
L_{\mathcal S}^{\mathrm{disc}}(x)
=
\lambda\sum_{i=1}^n\sum_{a\in D_i}\mathbf 1_{x_i\neq a}
\;-\;
\sum_{1\le i<j\le n} w_{ij}\mathbf 1_{x_i\neq x_j}.
\]
The unary wells satisfy
\[
U_i^{\mathrm{disc}}(t):=\sum_{a\in D_i}\mathbf 1_{t\neq a}
=
\begin{cases}
|D_i|-1, & t\in D_i,\\
|D_i|, & t\notin D_i.
\end{cases}
\]
Hence replacing any \(x_i\notin D_i\) by a value \(b\in D_i\) decreases the positive part by exactly \(\lambda\), while the total increase in the negative part is at most \(\sum_j w_{ij}\le \Delta\). Since \(\lambda=1+\Delta\), the same snapping argument goes through. Therefore the conclusions of Theorem~\ref{thm:all-different} and Corollary~\ref{cor:all-different-csp} remain valid for the discrete metric as well. The \(p\)-adic case remains the main one because it realises the same indicator behaviour inside the affine \(p\)-adic regression framework used throughout the paper.

\paragraph{Boolean clause rewards.}
Let
\[
\Phi=\bigwedge_{r=1}^{m} C_r
\]
be a CNF formula on variables \(z_1,\dots,z_n\), with every clause of width at most \(k\), and encode truth values by
\[
x_i=0 \iff z_i \text{ is true},
\qquad
x_i=1 \iff z_i \text{ is false}.
\]
For a literal \(\lambda\) on variable \(z_i\), define its failure indicator by
\[
\delta(\lambda;x)=
\begin{cases}
x_i, & \lambda=z_i,\\
1-x_i, & \lambda=\neg z_i.
\end{cases}
\]
For a clause \(C=(\lambda_1\vee\cdots\vee\lambda_s)\), where \(s\leq k\),
\[
C(x)\text{ is false}
\iff
\sum_{j=1}^{s}\delta(\lambda_j;x)=s.
\]
After moving constants to the right-hand side, this means that there are \(u_C\in\mathbb Z^n\) and \(t_C\in\mathbb Z\) such that
\[
C(x)\text{ is true}
\iff
u_C^\top x\neq t_C.
\]
Here \(u_C\) records the \(\pm 1\) literal coefficients, and \(t_C=s-\nu_C\), where \(\nu_C\) is the number of negated literals in the clause. The translation is mechanical: a positive literal contributes \(+x_i\), a negated literal contributes \(-x_i\), and the forbidden right-hand side is \(s-\nu_C\). For example,
\[
(z_1\vee z_2\vee \neg z_3)
\leadsto
x_1+x_2-x_3\neq 2.
\]
If \(p>k\) and \(x\in\{0,1\}^n\), then \(u_C^\top x-t_C\) is \(0\) when \(C\) is false and belongs to \(\{-1,\dots,-s\}\) when \(C\) is true, so
\[
|u_C^\top x-t_C|_p=\mathbf 1_{C(x)\text{ is true}}.
\]
Thus a negative observation \((-1,u_C,t_C,1)\) rewards satisfaction of the clause. If \(\Delta\) is the maximum number of clauses containing any variable and \(\alpha>\Delta\), then
\[
L_\Phi(x)
=
\alpha\sum_{i=1}^{n}\bigl(|x_i|_p+|x_i-1|_p\bigr)
-\sum_{C\in\Phi}|u_C^\top x-t_C|_p
\]
has the same snapping property as before: every global minimiser lies in \(\{0,1\}^n\), and on Boolean assignments
\[
L_\Phi(x)=\alpha n-\operatorname{sat}_\Phi(x).
\]
So the global minimisers are exactly the assignments that maximise the number of satisfied clauses. Unlike the earlier ordinary summed \(2\)-adic \textsc{Max-Cut} reduction \cite{baker2026nphard}, this construction uses a signed objective; its \(3\)-CNF specialisation assumes \(p>3\).

\paragraph{General false labels.}
Nothing above requires the false value to be \(1\) for every variable. Encode \(z_i=\mathrm{false}\) by any \(b_i\equiv 1\pmod p\), not necessarily distinct, keep the \(\pm 1\) literal coefficients, take as the forbidden target the sum of the positive literals' labels, and replace the Boolean wells by \(|x_i|_p+|x_i-b_i|_p\). On Boolean-labelled assignments the clause residual equals minus the sum of the false labels of the satisfied literals: zero when the clause is falsified, and otherwise a sum of at most \(k<p\) labels congruent to \(1\), hence a unit. The snapping and domination arguments are unchanged because every integer norm is at most \(1\). The uniform label \(b_i=1\) is the degenerate member of this family, and the experiments in this paper use it; Section~\ref{sec:general-labels} describes the member \(b_i=1+p\,2^{i}\), which leaves the loss landscape identical while making each residual value identify the satisfied literals of its clause.

\begin{corollary}[NP-hardness at a fixed prime]
\label{cor:signed-nphard}
The threshold decision problem for signed affine \(p\)-adic residual objectives is NP-hard already for the fixed prime \(p=5\); consequently, so is global minimisation.
\end{corollary}

\begin{proof}
Apply the construction above to a \(3\)-CNF formula, with \(p=5\) and \(\alpha=\Delta+1\). Its minimum is at most the threshold \(\alpha n-m\) if and only if all \(m\) clauses can be satisfied. This is a polynomial-time reduction from \textsc{3-SAT} to the threshold problem.
\end{proof}

\paragraph{Signed versus positive complement encodings.}
The signed construction is more compact rather than uniquely expressive. For the uniform label \(b_i=1\), the same clause can be encoded positively through its failure sum
\[
s_C(x):=\sum_{j=1}^{s}\delta(\lambda_j;x)
\]
and the allowed set \(S_C=\{0,\dots,s-1\}\). For \(S\subseteq\{0,\dots,s\}\), set
\[
W_S(t)=\sum_{a\in S}|t-a|_p.
\]
When \(p>s\) and \(t\in\{0,\dots,s\}\), the positive term \(W_{S_C}(t)\) equals \(s-1\) on the satisfying values \(0,\dots,s-1\) and \(s\) on the forbidden value \(s\). It therefore encodes exactly the same clause penalty using \(s\) positive observations. By contrast, the signed encoding uses one negative observation for the single forbidden affine equality. Filling the missing positive well at \(s\) makes the positive sum constant, which illustrates the complement relationship. The contribution of the signed form is the direct one-row representation of a small forbidden set, together with the coordinatewise domination theorem that keeps its negative reward bounded and snaps variables to their domains.

\subsection{Comparison with the Archimedean case}
\label{sec:archimedean-comparison}

The hypotheses of Theorem~\ref{thm:compiler-template} are genuinely
non-Archimedean. If we try to transplant the construction to \(\mathbb{R}\) by
scoring the same signed data with squared residuals,
\begin{equation*}
L_{\mathbb{R}}(x) \;=\; \sum_{\ell} \sigma_\ell\,\gamma_\ell\,
\bigl(u_\ell^{\top}x - b_\ell\bigr)^{2},
\qquad x \in \mathbb{R}^{n}
\end{equation*}
it fails for two independent reasons.

\paragraph{First obstruction: unboundedness.}
\(L_{\mathbb{R}}\) is a quadratic function whose leading form is
$x^{\top}Qx$ with
$Q = \sum_{\ell}\sigma_\ell\gamma_\ell\,u_\ell u_\ell^{\top}$. Such a
function is bounded below on $\mathbb{R}^{n}$ if and only if $Q$ is
positive semidefinite and the linear part vanishes on $\ker Q$.
Boundedness is therefore governed by global spectral and null-space
conditions on the entire design, rather than by a termwise bound on the
residuals, and a sufficiently weighted negative row makes $Q$
indefinite. In particular, the coordinatewise domination hypothesis of
Theorem~\ref{thm:compiler-template} is unsound over $\mathbb{R}$. Take
$n=2$, pin each variable to the singleton domain $\{0\}$ with weight
$\lambda$, and reward their inequality with one negative row:
\begin{equation*}
L_{\mathbb{R}}(x_1,x_2) \;=\; \lambda x_1^{2} + \lambda x_2^{2}
\;-\; \gamma\,(x_1-x_2)^{2}.
\end{equation*}
Along the line $x_1=-x_2=t$ this equals $2(\lambda-2\gamma)\,t^{2}$, so
the choice $\lambda=3$, $\gamma=2$ satisfies the domination condition
$\lambda>\gamma$ while $L_{\mathbb{R}}\to-\infty$; recentring the wells
at nonzero values changes only lower-order terms. (The $p$-adic
compilation of the same contradictory two-variable system is bounded and
returns its minimum-conflict assignment, as in
Section~\ref{sec:minimum-conflicts-potts}.) Adding linear inequality
constraints to make the real problem well posed does not remove the
computational obstruction: minimising a quadratic function with a single
negative eigenvalue under linear inequality constraints is NP-hard
\cite{PardalosVavasis1991}.

\paragraph{Second obstruction: no exact pinning.}
Convexity defeats the pinning rows before any negative weight appears.
Score the pinning rows of a Boolean coordinate, \(D_i=\{0,1\}\), with
squared residuals: the unary well becomes
$\lambda\bigl(t^{2}+(t-1)^{2}\bigr)$, minimised at the mean
$t=\tfrac12$ --- in neither well. Increasing $\lambda$ simply keeps $t$
closer to $\tfrac12$, as perturbations from other data points have less
effect. The same holds for any finite domain: the convex analogue
$\sum_{a\in D_i}(t-a)^{2}$ of the unary well is minimised at the mean
of $D_i$, rather than precisely on every point of $D_i$. More
generally, the minimiser set of a convex function is convex, so no
convex separable penalty can pin a real variable exactly to a finite
non-singleton domain. Even where the real penalty method does work ---
a singleton domain, i.e.\ an equality constraint --- the quadratic
penalty is only asymptotically exact: at any finite weight the
remaining rows drag the minimiser off the constraint, and exactness
requires $\lambda\to\infty$. The domination threshold of
Theorem~\ref{thm:compiler-template} is, by contrast, a finite weight
that pins exactly. Exact pinning over $\mathbb{R}$ already requires
nonconvex multiwell penalties before any negative weight appears.

Over $\mathbb{Z}_p$ both obstructions disappear at once. Each residual
norm lies in $[0,1]$, so for every $x\in\mathbb{Z}_p^{n}$,
\begin{equation*}
-\sum_{\ell:\,\sigma_\ell=-1}\gamma_\ell
\;\le\; L(x) \;\le\; \sum_{\ell:\,\sigma_\ell=+1}\gamma_\ell,
\end{equation*}
a two-sided bound that holds term by term, for every configuration of
signs and weights: a negative row is a bounded reward, not a source of
indefiniteness. At the same time the unary sum
$\sum_{a\in D_i}|t-a|_p$ is minimised exactly on $D_i$
(Lemma~\ref{lem:unary-well}), so
multiwell pinning is available inside the affine residual class itself.
These two facts are what make the coordinatewise condition of
Theorem~\ref{thm:compiler-template} an exact penalty: the strong triangle
inequality caps the effect of any single-coordinate change on every
interaction row, so a pinning weight above the incident interaction
weight dominates all rewards simultaneously.

Unlike $\mathbb{R}$, away from its zero hyperplane
the map $x\mapsto|u^{\top}x-b|_p$ is locally constant, so the loss
offers no gradient and no notion of a residual being ``nearly'' zero;
this is why Zubarev proposes random-walk rather than gradient dynamics
for $p$-adic losses \cite{zubarev2025padic}. The boundedness that licenses
signed weights and the flatness that starves local search are two faces
of the same non-Archimedean absolute value;
Appendix~\ref{app:powers-two} analyses a concrete instance of the latter
for the powers-of-two encoding.

\section{Unsatisfiable instances}
\label{sec:minimum-conflicts-potts}

Consider what happens when the instance has no proper list-colouring: there the global minimisers are the domain-respecting assignments of least conflict count. For example, take the triangle \(K_3\) with all three lists equal to \(\{1,2\}\). No proper \(2\)-colouring of \(K_3\) exists, so every domain-respecting assignment leaves at least one edge monochromatic. Here \(\Delta=2\), so \(\lambda=3\), and by Theorem~\ref{thm:all-different} the global minimisers are exactly the assignments with a single conflict, at loss \(\lambda\sum_i(|D_i|-1)-|E|+1=3\cdot 3-3+1=7\). This Max-CSP regime is distinct from the satisfiable case.

If a best-effort solution should treat some constraints as more important than others, assign them correspondingly higher or lower losses. Replacing the unit reward on an edge \(e\) by a weight \(w_e>0\), and raising each pinning weight \(\lambda_i\) above the total weight incident on coordinate \(i\), gives the weighted minimum-conflict objective by the same domination argument.

\section{Sudoku as a special case}
\label{sec:sudoku}
Sudoku has standard exact-cover, constraint-programming, and SAT formulations \cite{knuth2000dancing,simonis2005sudoku,lynce2006sudoku}, and puzzle variants have their own complexity theory \cite{yato-seta}. The construction below is not meant to compete with those solvers. Sudoku is used because its variables, domains, and all-different constraints are small enough to write out explicitly. We then perform experiments with stochastic optimisation methods from the \(p\)-adic regression literature on a small sample of randomly carved Sudoku instances.

\subsection{Sudoku}
A (standard) Sudoku puzzle consists of a partially filled $9\times 9$ grid. The task is to assign a digit in $\{1,\dots,9\}$ to each empty cell such that:
\begin{enumerate}
    \item each row contains all digits $1,\dots,9$ exactly once;
    \item each column contains all digits $1,\dots,9$ exactly once; and
    \item each $3\times 3$ box contains all digits $1,\dots,9$ exactly once.
\end{enumerate}
Equivalently, in each unit (row, column, box) all entries are \emph{pairwise distinct} and each entry lies in $\{1,\dots,9\}$.

Standard Sudoku fits Corollary~\ref{cor:all-different-csp} directly. The vertices are the $81$ cells, two vertices are adjacent when the corresponding cells share a row, column, or $3\times 3$ box, unclued cells have domain $\{1,\dots,9\}$, and clue cells have singleton domains. A proper list-colouring of this graph is exactly a valid Sudoku completion.

The Sudoku facts below are direct corollaries of the unary-well and edge-indicator lemmas. We state them before writing the concrete objectives.

\subsection{Positive pinning: snapping coefficients to digits}

Let $D=\{1,2,\dots,9\}\subset \mathbb{Z}_p$, and write the Sudoku variables as a vector
\[
x \in \mathbb{Z}_p^{81},
\]
with one coefficient per cell.

\begin{definition}[Sudoku digit-snapping penalty]
For $p>9$, define the digit-snapping penalty
\[
U_D(t) := \sum_{k=1}^9 |t-k|_p.
\]
\end{definition}

By Lemma~\ref{lem:unary-well} applied to the common domain \(D\), one has \(U_D(t)\ge 8\) for every \(t\in\mathbb Z_p\), with equality if and only if \(t\in D\).

The unary term is the nine-valued specialisation of the arbitrary-residue positive wells in Mihara's general forcing theorem \cite{mihara2026forcing}: the observations $t=k$ for all $k\in D$ create a multiwell objective with minima exactly at $D$. The additional signed terms below reward pairwise inequality and produce the Sudoku/list-colouring objective.

\subsection{Negative signed residual terms: rewarding inequality via p-adic norms}

The remaining Sudoku constraints are that entries in each row/column/box are pairwise unequal. Pairwise inequality can be encoded using differences $x_i-x_j$.

\begin{corollary}[Sudoku differences as inequality indicators]
\label{cor:sudoku-differences}
Let $p>9$ and $a,b\in D$. Then
\[
|a-b|_p =
\begin{cases}
0 & a=b,\\
1 & a\neq b.
\end{cases}
\]
\end{corollary}

\begin{proof}
This is Lemma~\ref{lem:edge-indicator} with \(q=9\).
\end{proof}

\subsection{Minimal and expanded residual objectives}
Let $\mathcal{P}$ be the set of cell-pairs $(i,j)$ that share an exclusion constraint (same row, same column, or same box). Let $\mathcal{C}$ be the set of given clues; for $i\in\mathcal{C}$ let $g_i\in D$ be the fixed digit.

Specialising Definitions~\ref{def:list-dataset} and \ref{def:list-loss}, together with Theorem~\ref{thm:all-different}, to the Sudoku graph gives the minimal synthetic dataset
\begin{align}
L_{\min}(x) &=
\alpha\sum_{i\notin\mathcal{C}}\sum_{k=1}^9 |x_i-k|_p
\;+\;
\alpha\sum_{i\in\mathcal{C}} |x_i-g_i|_p
\;-\;
\sum_{(i,j)\in\mathcal{P}} |x_i-x_j|_p.
\label{eq:objective-minimal}
\end{align}
A separate coefficient on the pairwise term is unnecessary here: any positive edge weight can be scaled out, so we normalise that coefficient to \(1\).
Each cell appears in 20 pairs: the 8 row disequalities, the 8 column disequalities and the other 4 disequalities in its box. (See Figure~\ref{fig:twenty}.)

\begin{figure}[t]
\centering
\includegraphics[width=0.80\linewidth]{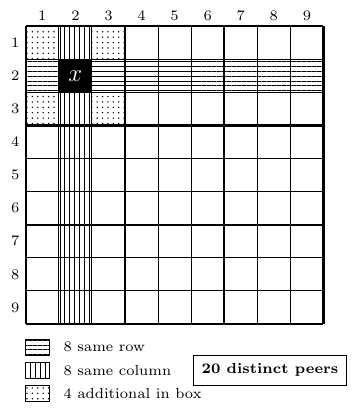}
\caption{The \(20\) distinct peers of cell \(x\): eight cells in the same row, eight in the same column, and four additional cells in the same \(3\times3\) box. }
\label{fig:twenty}
\end{figure}

Thus \(\alpha\) itself can play the role of the general pinning weight \(\lambda\), and any choice \(\alpha>20\) satisfies the domination condition from Theorem~\ref{thm:all-different}. The minimal dataset therefore contains \(9(81-|\mathcal C|)+|\mathcal C|=729-8|\mathcal C|\) positive observations and \(810\) negative observations. This is the formulation used in the correctness statements below.

For exposition, we also use the redundant expanded objective
\begin{align}
L(x) &=
\alpha\sum_{i=1}^{81}\sum_{k=1}^9 |x_i - k|_p
\;+\;
\alpha_{\mathrm{clue}}\sum_{i\in\mathcal{C}} |x_i - g_i|_p
\;-\;
\sum_{(i,j)\in\mathcal{P}} |x_i-x_j|_p.
\label{eq:objective}
\end{align}
Equation \eqref{eq:objective} keeps the nine digit-well residuals visible even on clue cells and then adds the clue pins separately. It therefore contains \(729+|\mathcal C|\) positive observations and the same \(810\) negative observations. The first two sums are positive pinning terms, and the final sum is a negative pairwise reward term. Equivalently, counting peer pairs per unit gives $27\binom{9}{2}=972$ negative terms, but $162$ of these pairs lie simultaneously in a row/column \emph{and} a box, leaving $810$ distinct pairs overall.

This expanded form needs an explicit clue-weight assumption. On a clue cell, the first sum in \eqref{eq:objective} has the same value \(8\alpha\) for every digit \(x_i\in D\), so clue fidelity is enforced only by \(\alpha_{\mathrm{clue}}|x_i-g_i|_p\). The same domination argument applies to \eqref{eq:objective} when \(\alpha_{\mathrm{clue}}>20\); the simplest choice is \(\alpha_{\mathrm{clue}}=\alpha=21\). Under that extra condition, \eqref{eq:objective} is a redundant expanded form of \eqref{eq:objective-minimal}.

\paragraph{Design choice: clues as data.}
A conventional discrete-optimisation model would substitute the clue digits and optimise over the unclued cells alone. Both objectives instead keep all \(81\) coordinates and enforce clues through heavily weighted unary rows, because the paper frames completion as regression: the coefficient space \(\mathbb Z_p^{81}\) and the row schema are then fixed once for all puzzles, and editing a clue changes observations rather than the variable space, just as changing a dataset does not change a regression model class. Nothing is lost by this choice: the domination theorem absorbs the substitution view as the special case of singleton domains, and on clue-consistent states both formulations induce the same conflict count. The expanded form \eqref{eq:objective} makes the reading visible by separating the puzzle-independent digit wells from the clue rows.

Once \(x\) is digit-valued, the pairwise part is simply
\[
-\sum_{(i,j)\in\mathcal{P}} |x_i-x_j|_p
=
-|\mathcal P|+\sum_{(i,j)\in\mathcal{P}}\mathbf 1_{x_i=x_j}.
\]
For the minimal objective on domain-respecting states, and for the expanded objective on digit-valued, clue-consistent states, the remaining terms are constant. In those regimes the loss differs from the \emph{deduplicated} peer-conflict count only by an additive constant. The heuristics below do not optimise that deduplicated peer objective directly: on the row-permutation state space they use a duplicated unit-scope column/box surrogate that counts conflicts separately inside each column and each box. That surrogate is described in Section~\ref{sec:illustrative-computations} and Appendix~\ref{app:heuristics}.

\begin{corollary}[Sudoku completion as a special case]
\label{cor:sudoku-special-case}
Let \(G_{\mathrm{Sud}}\) be the graph on the \(81\) cells of a Sudoku puzzle, with an edge between two cells exactly when they share a row, column, or \(3\times 3\) box. For a puzzle with clue set \(\mathcal{C}\) and clue digits \(g_i\), define
\[
D_i :=
\begin{cases}
\{g_i\}, & i\in \mathcal{C},\\
\{1,\dots,9\}, & i\notin \mathcal{C}.
\end{cases}
\]
Assume \(\alpha>20\). Then every global minimiser of the minimal objective \eqref{eq:objective-minimal} is a clue-respecting digit assignment that minimises the peer-conflict count. In particular, if the Sudoku instance is satisfiable, the global minimisers are exactly the valid completions of the puzzle.
\end{corollary}

\begin{proof}
Each row, column, and box of Sudoku forms a clique of size \(9\) in \(G_{\mathrm{Sud}}\). Thus a proper list-colouring assigns pairwise distinct digits in each unit. Since the available values are exactly \(\{1,\dots,9\}\), pairwise distinctness in a unit is equivalent to containing every digit exactly once. The singleton domains enforce the clues. Apply Theorem~\ref{thm:all-different}.
\end{proof}

In the deduplicated peer graph, every cell has degree \(20\), so Theorem~\ref{thm:all-different} requires \(\alpha>20\) for the minimal objective \eqref{eq:objective-minimal}; \(\alpha=21\) is enough. For the expanded objective \eqref{eq:objective}, the same claim requires \(\alpha_{\mathrm{clue}}>20\), because the all-digit wells are constant on clue cells. If one instead counts row, column, and box overlaps separately, each cell has \(8+8+8=24\) incident negative terms, and the safe unary bound becomes \(>24\).

\section{Representational parsimony}
\label{sec:parsimony}

A standard SAT or exact-cover lift introduces Boolean variables
\[
y_{r,c,d}\in\{0,1\},
\]
with the intended meaning that cell $(r,c)$ contains digit $d$. This produces \(9\times 9\times 9=729\) indicators before any clues are applied. The p-adic construction instead uses an \(81\)-cell digit space. This variable-level comparison is the main parsimony claim: one digit-valued parameter per semantic object rather than nine Boolean indicators per cell.

The minimal objective \eqref{eq:objective-minimal} is already a single scalar loss on that state space. The theorem-side expanded objective \eqref{eq:objective} contains \(729\) digit-snapping residuals, \(|\mathcal{C}|\) clue residuals, and \(810\) pairwise residuals. A structured constraint-programming model with \(81\) cell variables and \(27\) \emph{alldifferent} constraints is already semantically close to the puzzle, so the comparison here is not based on raw constraint count alone. The signed $p$-adic construction packages the all-different structure into one residual objective on the \(81\) cell variables.

Appendix~\ref{app:powers-two} records an alternative encoding that maps digits to powers of two so that each row, column, or box contributes a single sum residual rather than pairwise terms. In local-search experiments on \(50\) benchmark puzzles, that encoding produced no successful solves: the search repeatedly stalled on loss plateaus.

\section{Computations with the induced \texorpdfstring{\(p\)-adic}{p-adic} loss}
\label{sec:illustrative-computations}

We implemented two finite-state local-search solvers: a greedy row-swap search and a discrete Zubarev walk. We also implemented a Mihara-inspired last-digit (modulo-\(p\)) equality fit as a diagnostic comparison. Appendix~\ref{app:heuristics} defines the solver loss, records the two algorithms and a locality lemma for row swaps, and gives traces on a standard puzzle.

The greedy solver simply counts integer conflicts: it does not evaluate a \(p\)-adic valuation numerically. On the digit-restricted state space this count is an exact integer representative of the induced zero-or-one residual loss, with column--box overlaps counted twice. The Zubarev solver uses the same induced loss through the Boltzmann transition rule \(\exp(-\beta\,\Delta L)\); its finite implementation can likewise compute \(\Delta L\) from integer conflicts because all relevant norms have already collapsed to indicators. The Mihara-inspired comparison instead performs modulo-\(19\) last-digit equality fitting on a positive-only expansion: every negative peer row is replaced by the sixteen equalities for nonzero digit differences, and unary-row weights are preserved in the consensus score.

On a small sample of \(18\) randomly carved puzzles, the greedy and Zubarev solvers both completed every puzzle, giving \(36\) successful method--puzzle runs. The row-swap search used fewer steps on \(14\) of the \(18\) paired puzzles. Its median step counts were \(537.5\), \(2784\), and \(7040\) at \(36\), \(30\), and \(26\) clues, versus \(4533.5\), \(8751.5\), and \(9911\) for the Zubarev walk. Appendix~\ref{app:heuristics-results} gives the experimental setup, full table, and a representative loss curve. These proof-of-concept experiments show only that both finite-state procedures can reach completions in this sample; they are not a solver benchmark. The Mihara-inspired solver is not included among the \(36\) solver runs, since it did not successfully solve any of the sample puzzles.

The solver is available at \url{https://padic-logic.symmachus.org/} (with source code at \url{https://github.com/solresol/padic-logic}); a screenshot appears in Figure~\ref{fig:browser-sudoku-solution}. The code that ran these experiments is at
\url{https://github.com/solresol/sudoku-padic-regression}.

\begin{figure}[H]
\centering
\includegraphics[width=\linewidth]{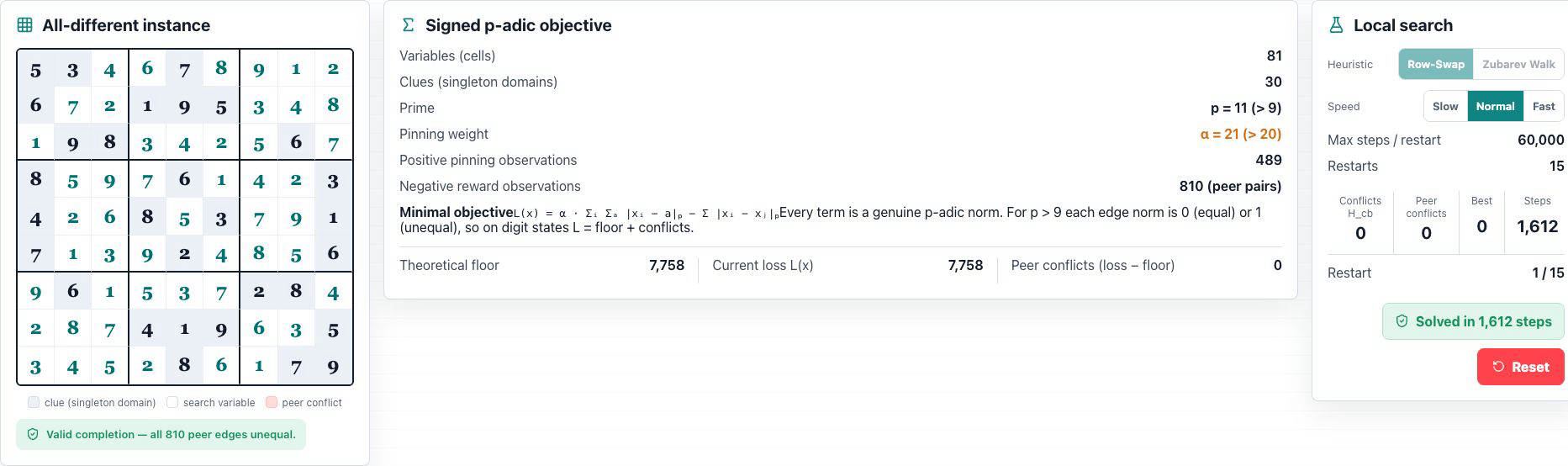}
\caption{Screenshot of \url{https://padic-logic.symmachus.org/}: the Sudoku route after a row-swap run reaches a valid completion. The native \(81\)-cell grid, the \(1{,}299\)-row signed objective summary, the theoretical-floor equality, and the zero-conflict search result.}
\label{fig:browser-sudoku-solution}
\end{figure}

\section{Related work}
\label{sec:related-work}

From the constraint-programming side, the finite-domain objective is a valued-CSP or Max-CSP penalty function \cite{schiex1995vcsp}. Replacing a global \emph{alldifferent} constraint by a clique of pairwise inequalities preserves its solutions but not the stronger propagation of a matching-based global constraint \cite{regin1994alldifferent}. On the finite domain, the list-colouring objective is the zero-temperature antiferromagnetic Potts energy with local list constraints \cite{ellismonaghan2014potts}; related discrete problems are often mapped to Ising or QUBO objectives \cite{lucas2014ising}. The local-search experiments are closest algorithmically to min-conflicts repair search \cite{minton1990minconflicts} and to Lewis's metaheuristic Sudoku study \cite{lewis2007metaheuristics}, whose simulated annealing keeps each \(3\times3\) box as a permutation of \(1,\dots,9\) and swaps two non-clue cells within a box; the heuristics in Appendix~\ref{app:heuristics} make the analogous choice with rows as the permutation unit. Sudoku itself is routinely formulated using exact cover, SAT, or \emph{alldifferent} constraints \cite{knuth2000dancing,simonis2005sudoku,lynce2006sudoku}. Martins gives a broader recent account of learning primitives over \(p\)-adic spaces \cite{martins2025learning}; the present construction specialises instead to signed affine objectives that compile finite-domain constraints.

\section{Conclusion}

Signed \(p\)-adic residuals have a direct finite-domain role. Positive synthetic terms pin variables to a finite set; after that, negative terms are bounded rewards for inequality or clause satisfaction.

Sudoku is a finite-domain case study in which the variables, domains, and all-different constraints are explicit. All-different systems admit polynomial-time encodings whose global minimisers are the domain-respecting minimum-conflict assignments, and in the satisfiable case exactly the satisfying assignments.

The same signed-residual construction is not limited to pairwise inequality. After Boolean pinning, every clause of width less than \(p\) becomes one forbidden affine equality, so a single signed term rewards its satisfaction. A positive-only affine encoding can represent the same finite-domain penalty by listing all allowed residual values; the signed formulation's advantage is the compact one-row representation of the forbidden complement within that affine class. Section~\ref{sec:general-labels} shows how variables can receive distinct false labels whose residual values identify the satisfied literals without changing the loss landscape. No optimisation or verification advantage is claimed for these labellings.

The companion site, \url{https://padic-logic.symmachus.org}, lets readers trace a Boolean CSP from clauses to forbidden affine equalities and a signed residual dataframe, or compile an editable Sudoku into its \(1{,}299\)-row objective and watch search reach the theoretical floor.

\appendix

\section{Mihara's digitwise regression and positive complement expansion}
\label{app:mihara-digitwise}

Mihara's forcing theorem \cite{mihara2026forcing} considers a product \(C=\prod_i B_i\) of residue sets, a nonnegative base objective satisfying a concentration condition near \(C\), and a positive unary forcing term whose weight \(M\) exceeds the base objective's range on \(C\). The forced objective has the same optima on \(C\); specialising the construction to affine residual data proves ordinary positive \(p\)-adic linear regression NP-hard for every fixed prime.

\begin{table}[H]
\centering
\small
\begin{tabular}{@{}p{0.20\linewidth}p{0.36\linewidth}p{0.36\linewidth}@{}}
\toprule
 & Mihara's forcing theorem & Signed affine compiler here \\
\midrule
Base terms & Nonnegative \(p\)-adic optimisation concentrated near a product of residue sets & Signed affine residuals that are zero or units on the target finite domain \\
Pinning condition & One forcing scale \(M\) above the base objective's range on the finite set & Coordinatewise \(\lambda_i\) above the absolute incident interaction weight \\
Off-domain control & Concentration plus positive unary forcing & Ultrametric Lipschitz bound for positive and negative interaction rows \\
Finite-domain role & Preserve the base objective's optima on \(C\) & Realise equality, inequality, and forbidden-hyperplane indicators exactly \\
Main consequence & Positive linear regression NP-hard for each fixed prime & Compact complement encoding and exact minimum-conflict characterisation \\
\bottomrule
\end{tabular}
\caption{Comparison with Mihara's generalisation of Baker forcing \cite{mihara2026forcing}. The two results share finite-domain pinning but use different hypotheses to control the remaining objective.}
\label{tab:forcing-comparison}
\end{table}

Mihara's linear-regression algorithm is described in \cite{mihara2026digitwise}. For his digitwise regression algorithm, the data are samples
\[
(\vec x_i,y_i)\in\mathbb Z_p^D\times\mathbb Z_p
\]
from a single hidden affine graph
\[
y=\langle c,\vec x\rangle,
\]
with digitwise noise. Writing
\[
I_e=\{i:y_i-\langle c,\vec x_i\rangle\in p^e\mathbb Z_p\},
\]
the method estimates \(c\bmod p^E\) recursively: first solve a modulo-\(p\) linear-regression problem to estimate \(c\bmod p\), subtract that digit, divide residuals by \(p\), restrict to the samples in \(I_1\), and repeat. The probabilistic argument relies on a large, sufficiently random sample from the hidden affine graph, a small digitwise corruption rate such as
\[
\frac{|I_e\setminus I_{e+1}|}{|I_e|}\le r \ll \frac12,
\]
and the condition that the noise-free samples have the right affine hull modulo \(p\) at each digit step.

Each finite-field elimination, consistency check, and dataset scan between random draws is polynomial-time. The complete published procedure is not worst-case polynomial-time: Algorithm~6 has an outer \texttt{while True} restart loop, while its sampling subroutines have no polynomial bound on the number of random draws. Mihara explicitly notes that the process need not terminate and reports one \(D=100\), \(r=0.1\) experiment in which it had not terminated after \(1{,}700\) initialisations. The probability analysis gives useful expected behaviour when the random-sampling, affine-hull, and noise assumptions hold; it does not provide a polynomial termination guarantee for arbitrary input data.

Those assumptions do not directly match the signed Sudoku construction. In this paper the unknowns \(x_i\) are the cell values themselves, not coefficients of an observed response law. The residual rows are exact synthetic constraints chosen by the modeller: unary wells pin cells to allowed digits, while negative residuals \(-|x_i-x_j|_p\) reward inequality across peer edges. After digit snapping, the problem is a finite-domain conflict-minimisation problem. There is no naturally sampled response vector \(Y\), no hidden coefficient vector \(c\), and no large noise-free affine locus whose graph is assumed to have generated the observations.

For a finite-domain negative row, however, there is an exact positive-only reformulation. Let
\[
S=\{u^\top x:x\in C\}
\]
be the attainable affine values on the forced finite domain \(C\), and let \(t\in S\) be the forbidden target. When distinct elements of \(S\) remain distinct modulo \(p\),
\[
\sum_{s\in S\setminus\{t\}} |u^\top x-s|_p
= (|S|-1)-|u^\top x-t|_p,
\qquad x\in C.
\]
Indeed, at \(u^\top x=t\) all \(|S|-1\) positive rows contribute one, whereas at an allowed value exactly one complementary row vanishes. Thus the positive family has exactly the same minimisers as the original negative row, differing only by the constant \(|S|-1\).

For a CNF clause, \(S_C=\{u_C^\top x:x\in\{0,1\}^n\}\) contains at most one more value than the clause width. Replacing the forbidden row \(u_C^\top x=t_C\) by the positive rows \(u_C^\top x=s\) for \(s\in S_C\setminus\{t_C\}\) is therefore a polynomial-size expansion. Together with the positive weighted Boolean wells, it gives Mihara's modulo-\(p\) stage a positive-only weighted-consensus dataset: a satisfying clause contributes one unit of inlier weight, while a failed clause contributes none. The original signed row is retained for reporting \(v_p(u_C^\top x-t_C)\), its norm, and its signed contribution to the displayed loss.

For a Sudoku peer row, \(S=\{-8,-7,\ldots,7,8\}\) and the forbidden target is \(0\). The positive complement therefore contains the sixteen equalities \(x_i-x_j=s\) with \(s\ne0\). Choosing \(p=19\) keeps those integer differences distinct modulo \(p\). On digit states an unequal peer satisfies exactly one complementary equality and an equal peer satisfies none.

The companion website uses the positive complement expansion in both its Boolean-CSP and Sudoku Mihara modes. For Sudoku it plots two histories: the non-increasing positive-consensus loss used to select fits and the original signed objective used only as an audit. A better consensus fit can still be off-domain or violate clues, and the signed audit need not decrease. The interface therefore continues fresh starts until a decoded vector satisfies the original constraints or the user stops it. It treats the procedure as an unbounded-restart heuristic rather than a polynomial-time solver. A method guaranteed to solve every compiled instance in polynomial time would decide the NP-hard threshold problem of Corollary~\ref{cor:signed-nphard}; Mihara's algorithm makes no such guarantee.

\section{Heuristic details}
\label{app:heuristics}

\subsection{Row-swap local search}
To keep local updates cheap, the greedy implementation does not evaluate \eqref{eq:objective} or any \(p\)-adic valuation numerically. Instead, on the digit-valued, clue-consistent row-permutation state space it minimises the duplicated unit-scope column/box conflict count
\[
H_{\mathrm{cb}}(x):=
\sum_{c=1}^{9} \operatorname{conf}(C_c;x)
\;+\;
\sum_{b=1}^{9} \operatorname{conf}(B_b;x),
\]
where \(C_c\) and \(B_b\) are the cells in column \(c\) and box \(b\), and
\[
\operatorname{conf}(U;x):=\sum_{d=1}^{9} \binom{\#\{u\in U:x_u=d\}}{2}.
\]
Because each row is kept as a permutation of \(\{1,\dots,9\}\), the corresponding row contribution is identically zero and is omitted. This duplicated unit-scope loss is close to, but not identical with, the deduplicated peer-conflict objective coming from \(\mathcal P\): a conflicting pair that lies simultaneously in a column and a box is counted twice. For the seed-\(0\) initialisation shown later in Figure~\ref{fig:seed-zero-init}, the deduplicated peer-conflict count is \(47\) whereas \(H_{\mathrm{cb}}=56\), because nine conflicting pairs lie in both a column and a box.

The heuristics work directly on the digit-valued, clue-consistent row-permutation state space and minimise \(H_{\mathrm{cb}}\). Concretely, we initialise each row as a random permutation of $1,\dots,9$ consistent with its clues, and then apply row-wise swaps that reduce \(H_{\mathrm{cb}}\). This is a block-local search on the integer state space. The state space and neighbourhood follow the standard metaheuristic Sudoku design of Lewis \cite{lewis2007metaheuristics}, with rows rather than boxes as the permutation unit; the search procedure itself is not a contribution of this paper.

\begin{lemma}[Locality of row swaps]
\label{lem:row-swap-locality}
Let \(x\) be a digit-valued, clue-consistent Sudoku state in which each row is a permutation of \(1,\dots,9\). Let \(\xi\) swap two non-clue entries in a single row \(r\), at columns \(c_1\neq c_2\). Then the change
\[
\Delta H_{\mathrm{cb}}(\xi):=H_{\mathrm{cb}}(x\oplus \xi)-H_{\mathrm{cb}}(x)
\]
is determined entirely by the conflict-pair contributions from columns \(c_1\) and \(c_2\), together with the one or two boxes containing \((r,c_1)\) and \((r,c_2)\). No other term of the digit-restricted objective changes.
\end{lemma}

\begin{proof}
On the row-permutation state space, the heuristic objective is exactly \(H_{\mathrm{cb}}\), so only column and box contributions matter. A within-row swap preserves the row permutation, so the omitted row contribution remains zero. Every cell outside row \(r\) is unchanged, and within row \(r\) only the two swapped positions change value. Hence only the two corresponding columns can change their conflict counts. Likewise, only the boxes containing those two cells can change. All other columns and boxes see the same multiset of digits before and after the swap, so their contributions are unchanged.
\end{proof}

\begin{algorithm}[t]
\caption{Row-swap local search}
\begin{algorithmic}[1]
\Require Puzzle with clues $\mathcal{C}$, max steps $T$, restarts $R$
\For{$r=1$ to $R$}
    \State Initialise each row as a permutation of $1,\dots,9$ consistent with its clues
    \For{$t=1$ to $T$}
        \If{no column/box conflicts} \State \Return solution \EndIf
        \State Choose a row involved in a conflict
        \State Find the best swap of two non-clue cells in that row (largest decrease in \(H_{\mathrm{cb}}\))
        \State Apply the swap (or a random swap if no improving swap exists)
    \EndFor
\EndFor
\State \Return best assignment found
\end{algorithmic}
\end{algorithm}

\subsection{Trace on a standard puzzle}
To make the row-swap heuristic concrete, Figure~\ref{fig:example-puzzle} shows the standard example puzzle used in many Sudoku tutorials. With seed \(0\) (restart \(0\)), the row-wise random initialisation used by the solver produces the filled grid in Figure~\ref{fig:seed-zero-init}; this state has \(56\) column/box conflict pairs. Table~\ref{tab:trace-stepwise} shows the first few swaps from that initial state.

Step \(2\) illustrates the locality claim. After Step \(1\), row \(6\) contains
\[
(7,3,1,8,2,5,4,9,6),
\]
and the chosen move swaps the entries in columns \(3\) and \(6\), exchanging \(1\) and \(5\). By Lemma~\ref{lem:row-swap-locality}, only column \(3\), column \(6\), the box on rows \(4\)--\(6\) and columns \(1\)--\(3\), and the box on rows \(4\)--\(6\) and columns \(4\)--\(6\) need to be inspected. Before the swap these four units contribute
\[
4+6+1+2=13
\]
conflict pairs; afterwards they contribute
\[
3+3+2+1=9.
\]
Thus \(\Delta H_{\mathrm{cb}}=-4\), exactly as reported in Table~\ref{tab:trace-stepwise}. For this instance the solver reaches a full solution after \(227\) swaps on the same restart.

\begin{figure}[t]
\centering
\includegraphics[width=6.2cm]{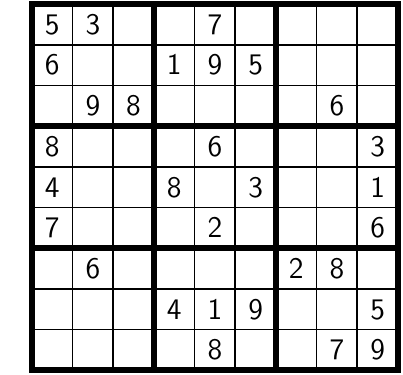}
\caption{Standard example puzzle used for the heuristic traces.}
\label{fig:example-puzzle}
\end{figure}

\begin{figure}[t]
\centering
\includegraphics[width=6.2cm]{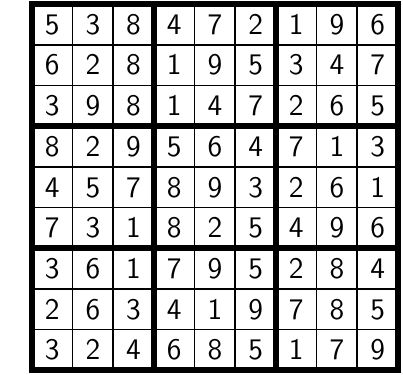}
\caption{Seed-\(0\) row-wise random initialisation (restart \(0\)), before any swaps.}
\label{fig:seed-zero-init}
\end{figure}

\begin{table}[t]
\centering
\small
\begin{tabular}{@{}rllllrr@{}}
\toprule
Step & Row & Swap (columns) & Swap (digits) & Type & Conflicts (before) & (after) \\
\midrule
1 & 5 & $7\leftrightarrow 8$ & $2\leftrightarrow 6$ & best   & 56 & 53 \\
2 & 6 & $3\leftrightarrow 6$ & $1\leftrightarrow 5$ & best   & 53 & 49 \\
3 & 2 & $2\leftrightarrow 3$ & $2\leftrightarrow 8$ & best   & 49 & 45 \\
4 & 4 & $7\leftrightarrow 6$ & $7\leftrightarrow 4$ & random & 45 & 47 \\
5 & 7 & $1\leftrightarrow 5$ & $3\leftrightarrow 9$ & best   & 47 & 40 \\
6 & 4 & $6\leftrightarrow 7$ & $7\leftrightarrow 4$ & best   & 40 & 38 \\
7 & 2 & $2\leftrightarrow 9$ & $8\leftrightarrow 7$ & best   & 38 & 36 \\
8 & 1 & $3\leftrightarrow 4$ & $8\leftrightarrow 4$ & best   & 36 & 35 \\
\bottomrule
\end{tabular}
\caption{First eight swap steps for the standard example puzzle (seed \(0\)). The table records both the swapped column indices and the digits occupying those columns before each move. Rows and columns are \(1\)-indexed.}
\label{tab:trace-stepwise}
\end{table}

\subsection{A discrete Zubarev walk}
The row-swap heuristic above is tailored to the strong-snapping Sudoku setting. It is also natural to test a more generic training dynamic from the p-adic regression literature. In \cite{zubarev2025padic}, Zubarev proposes a random-walk-based optimisation scheme for loss functions built from p-adic norms, precisely because such losses are locally constant almost everywhere and hence poorly served by gradient methods.

Concretely, the proposed training dynamics is a discrete-time random process
\begin{equation}
w_{t+1} = w_t + \xi_t(w_t,\beta_t),
\label{eq:zubarev-walk}
\end{equation}
where $\beta_t>0$ is a ``continuity'' (inverse temperature) parameter and $\xi_t$ is sampled from a distribution biased toward loss-decreasing moves. In Zubarev's formulation the conditional law has density (with respect to the $p$-adic Haar measure)
\begin{equation}
P(\xi \mid w,\beta)\ \propto\ \exp\!\bigl(-\beta\,(L(w+\xi)-L(w))\bigr).
\label{eq:zubarev-boltzmann}
\end{equation}
Under mild conditions (in particular, that the set of improving moves has nonzero measure), Zubarev shows there is a regime of $\beta$ values for which the expected loss decreases, i.e.\ $L(w_t)$ is a strict supermartingale~\cite{zubarev2025padic}.

In the digit-restricted implementation used here, the state is already digit-valued, clue-consistent, and row-permuted. We therefore treat the Zubarev walk as a \emph{finite-move} stochastic optimiser:
\begin{itemize}
    \item The parameter vector $w$ is the current Sudoku filling $x\in D^{81}$.
    \item The loss is the duplicated unit-scope column/box count \(H_{\mathrm{cb}}(x)\).
    \item A move $\xi$ is a swap of two non-clue cells within a single row (preserving each row as a permutation of $1,\dots,9$).
\end{itemize}
Replacing Haar measure and a continuous \(p\)-adic state space by counting measure on row swaps changes the hypotheses of Zubarev's result. No supermartingale or convergence guarantee is claimed for the finite heuristic below; the Boltzmann weighting is used only as a motivated move-selection rule.

For a chosen row $r$ and current state $x$, let $\Xi_r(x)$ be the set of all such swaps, and let \(\Delta H_{\mathrm{cb}}(\xi) := H_{\mathrm{cb}}(x\oplus\xi)-H_{\mathrm{cb}}(x)\) be the change in loss after applying swap \(\xi\) (where \(x\oplus\xi\) denotes ``apply the swap''). Replacing the Haar measure in \eqref{eq:zubarev-boltzmann} with the counting measure on the finite set $\Xi_r(x)$ gives the discrete sampling rule
\begin{equation}
\Pr(\xi \mid x,\beta_t) = \frac{\exp\!\bigl(-\beta_t\,\Delta H_{\mathrm{cb}}(\xi)\bigr)}{\sum_{\xi'\in\Xi_r(x)} \exp\!\bigl(-\beta_t\,\Delta H_{\mathrm{cb}}(\xi')\bigr)}.
\label{eq:zubarev-discrete}
\end{equation}
When $\beta_t\to 0$ this is uniform random swapping; when $\beta_t\to\infty$ it concentrates on the best (most loss-decreasing) swap, recovering a soft version of the greedy stepwise algorithm. In practice we use a simple annealing schedule: $\beta_t$ increases linearly from $\beta_0$ to $\beta_1$ over the allotted steps, gradually shifting from exploration to exploitation.

\begin{algorithm}[t]
\caption{Zubarev walk over row swaps (heuristic)}
\begin{algorithmic}[1]
\Require Puzzle with clues $\mathcal{C}$, max steps $T$, restarts $R$, schedule $\{\beta_t\}_{t=1}^T$
\For{$r=1$ to $R$}
    \State Initialise each row as a permutation of $1,\dots,9$ consistent with its clues
    \For{$t=1$ to $T$}
        \If{no column/box conflicts} \State \Return solution \EndIf
        \State Choose a row involved in a conflict
        \State Enumerate row swaps $\Xi_r(x)$ and compute $\Delta H_{\mathrm{cb}}(\xi)$ for each $\xi\in\Xi_r(x)$
        \State Sample $\xi\in\Xi_r(x)$ with probability $\propto \exp(-\beta_t\,\Delta H_{\mathrm{cb}}(\xi))$
        \State Apply swap $\xi$
    \EndFor
\EndFor
\State \Return best assignment found
\end{algorithmic}
\end{algorithm}

\subsection{Zubarev trace on the standard puzzle}
Using the same standard example puzzle, the Zubarev walk (seed \(0\), restart \(0\), \(\beta_0=0.5\), \(\beta_1=6.0\) linear schedule) starts from the same row-consistent initialisation with \(56\) column/box conflict pairs. Unlike the greedy stepwise method, the walk sometimes accepts loss-increasing or loss-neutral swaps early on; Table~\ref{tab:trace-zubarev} shows the first few moves. For this instance the Zubarev walk reaches a full solution after \(10{,}029\) swaps on the same restart.

\begin{table}[t]
\centering
\begin{tabular}{@{}rllrr@{}}
\toprule
Step & Row & Swap (columns) & Conflicts (before) & (after) \\
\midrule
1 & 5 & $2\leftrightarrow 8$ & 56 & 56 \\
2 & 6 & $7\leftrightarrow 8$ & 56 & 56 \\
3 & 2 & $3\leftrightarrow 7$ & 56 & 55 \\
4 & 8 & $2\leftrightarrow 8$ & 55 & 51 \\
5 & 4 & $2\leftrightarrow 3$ & 51 & 50 \\
6 & 4 & $6\leftrightarrow 7$ & 50 & 51 \\
7 & 2 & $7\leftrightarrow 8$ & 51 & 52 \\
8 & 8 & $3\leftrightarrow 7$ & 52 & 49 \\
\bottomrule
\end{tabular}
\caption{First eight Zubarev-walk swap steps for the standard example puzzle (seed \(0\)). Rows and columns are \(1\)-indexed.}
\label{tab:trace-zubarev}
\end{table}

\section{Illustrative experiments}
\label{app:heuristics-results}
We implemented both solvers above in JavaScript for the website and in pure Python; the experiment code is available at \url{https://github.com/solresol/sudoku-padic-regression}. The prime was set to $p=11$, but in the digit-restricted regime the relevant residual norms are exactly zero-or-one indicators. The greedy solver evaluates the resulting integer conflict count \(H_{\mathrm{cb}}\) directly. The Zubarev solver uses changes in the same induced loss in its Boltzmann probabilities; those changes can also be computed from \(H_{\mathrm{cb}}\) without numerically calling a valuation routine. Thus the \(p\)-adic construction determines the finite-state loss, while the implementation uses its exact integer representation.

\subsection{Experimental setup}
To obtain a quick, reproducible testbed, we generated random puzzles by:
\begin{enumerate}
    \item sampling a random solved grid; then
    \item carving it down to a specified clue count by removing entries uniformly at random.
\end{enumerate}
This procedure does \emph{not} enforce uniqueness of the resulting puzzle. Since Theorem~\ref{thm:all-different} characterises the entire set of global minimisers as the set of valid completions, a puzzle with several completions has several global minimisers, and the heuristics report whichever one their trajectory reaches first. Each carved puzzle was then solved with up to $T=60{,}000$ steps and $R=15$ random restarts. All results reported below use a fixed RNG seed ($123$) for reproducibility. The two methods receive the same \(18\) puzzle seeds and the same corresponding solve seeds, so their results are paired. There is only one solve seed per method--puzzle pair; the sample is not difficulty-normalised and supports no inferential performance claim.

The table can be regenerated from the repository root with the following commands:
\begin{quote}\small
\begin{verbatim}
python3 code/run_experiments.py --outdir outputs/paper_stepwise \
  --seed 123 --n 6 --clues 36,30,26 --max-steps 60000 \
  --restarts 15 --method stepwise
python3 code/run_experiments.py --outdir outputs/paper_zubarev \
  --seed 123 --n 6 --clues 36,30,26 --max-steps 60000 \
  --restarts 15 --method zubarev --beta0 0.5 --beta1 6.0 \
  --beta-schedule linear
\end{verbatim}
\end{quote}

\subsection{Results}
Table~\ref{tab:results} summarises \(36\) method--puzzle runs over \(18\) paired puzzles (six puzzles at each clue count, evaluated by both heuristics). Both methods succeeded on every puzzle. The row-swap method took fewer steps on \(14/18\) paired runs, including \(5/6\), \(6/6\), and \(3/6\) puzzles at \(36\), \(30\), and \(26\) clues respectively. Runtimes are machine-dependent and noisy, and the sample is too small for stronger algorithmic claims.

\begin{table}[t]
\centering
\begin{tabular}{@{}lrrrrr@{}}
\toprule
Method & Clues & Puzzles & Solved & Median steps & Median time (s) \\
\midrule
Stepwise & 36 & 6 & 6 & 537.5  & 0.270 \\
Stepwise & 30 & 6 & 6 & 2784 & 1.315 \\
Stepwise & 26 & 6 & 6 & 7040 & 3.990 \\
Zubarev walk & 36 & 6 & 6 & 4533.5 & 2.091 \\
Zubarev walk & 30 & 6 & 6 & 8751.5 & 3.951 \\
Zubarev walk & 26 & 6 & 6 & 9911 & 4.786 \\
\bottomrule
\end{tabular}
\caption{Descriptive proof-of-concept results on \(18\) paired, randomly carved puzzles (not uniqueness-checked), using $T=60{,}000$ steps and $R=15$ restarts. The Zubarev walk uses a linear schedule $\beta_0=0.5$ to $\beta_1=6.0$.}
\label{tab:results}
\end{table}

Figure~\ref{fig:loss} shows the conflict trajectory for the stepwise row-swap heuristic on the first \(30\)-clue table instance, with puzzle seed $115642242$ (solve seed $2738956839$), using the row-wise random initialiser described above, $T=60{,}000$, and $R=15$. The plotted run is restart $0$, and the vertical axis is the duplicated column/box conflict count \(H_{\mathrm{cb}}\). The corresponding trace puzzle, solution, CSV row, and figure source are all included in \path{outputs/paper_stepwise/} in the submission archive.

\begin{figure}[t]
\centering
\includegraphics[width=0.75\linewidth]{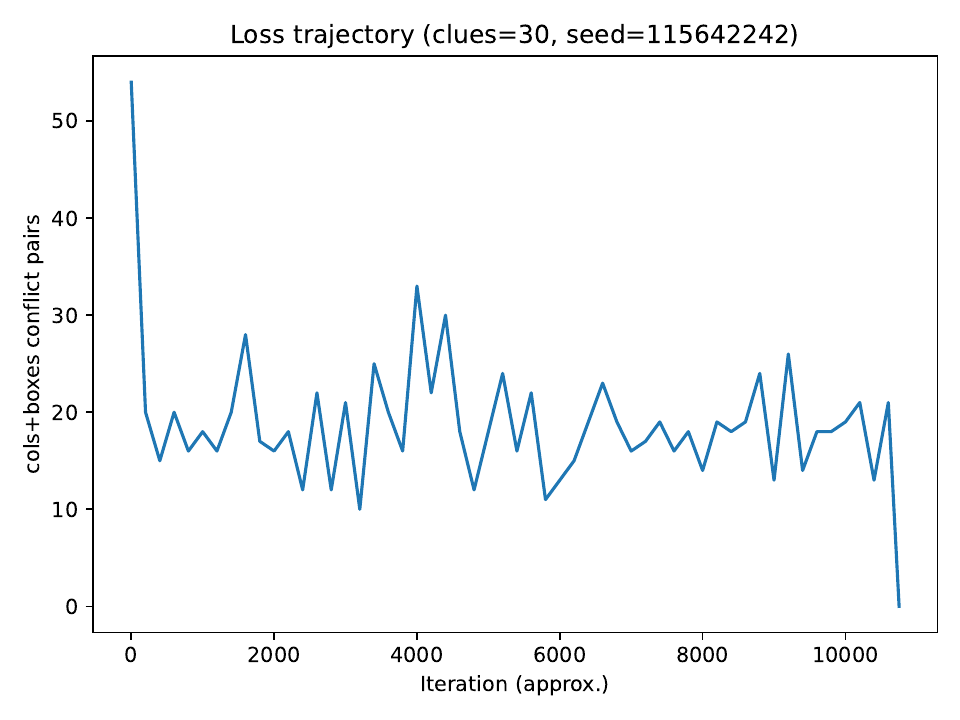}
\caption{Conflict trajectory for the stepwise row-swap heuristic on the paired-table $30$-clue instance with puzzle seed $115642242$ (solve seed $2738956839$), from the row-wise random initialisation on restart $0$.}
\label{fig:loss}
\end{figure}

\section{Powers-of-two experiments}
\label{app:powers-two}
An alternative maps digits \(d\in\{1,\dots,9\}\) to powers of two,
\[
d \longmapsto 2^{d-1} \in A := \{1,2,4,\dots,256\}.
\]
If every cell variable is restricted to \(A\), then the condition that a row, column, or box contains each digit exactly once is equivalent to a single sum constraint
\[
\sum_{(r,c)\in g} \beta_{r,c} = 1+2+\cdots+256 = 2^9-1 = 511,
\]
for each unit \(g\). This replaces the pairwise penalties by \(27\) residuals of the form
\[
r_g(\beta) := 511 - \sum_{(r,c)\in g}\beta_{r,c}.
\]
For a prime \(p\), one can then study the objective
\[
L_p(\beta) := \sum_g |r_g(\beta)|_p,
\]
optionally together with a regulariser that encourages \(\beta_{r,c}\in A\) and enforces the clues.

The archived local-search experiments used the \(50\) puzzles from Project Euler Problem \(96\). Table~\ref{tab:powers-two-results} summarises the outcomes. None of these runs found a valid completion.

\begin{table}[H]
\centering
\small
\begin{tabular}{@{}p{0.14\linewidth}p{0.44\linewidth}p{0.10\linewidth}p{0.24\linewidth}@{}}
\toprule
Experiment & Setting & Solved & Summary \\
\midrule
E1 & Greedy cell-swap on \(15\) primes, \(50\) puzzles, \(3\) random initialisations per puzzle--prime pair (\(2250\) runs) & \(0/2250\) & Mean lift \(0.003\); mean \(12.5\) steps \\
E2 greedy & Greedy cell-swap on the \(5\) strongest E1 primes, again with \(50\) puzzles and \(3\) initialisations (\(750\) runs) & \(0/750\) & Mean lift \(0.009\); mean \(19.3\) steps \\
E2 SA & Simulated annealing on the same \(750\) runs & \(0/750\) & Mean lift \(0.000\); mean \(5000\) steps \\
\bottomrule
\end{tabular}
\caption{Archived local-search results for the powers-of-two encoding. The E1 prime sweep used primes \(2,3,5,7,11,13,17,23,47,73,97,127,257,521,2311\).}
\label{tab:powers-two-results}
\end{table}

\paragraph{Why the encoding defeats local search.}
An Archimedean search on the same statistic would see a graded
signal: the deficit $|511 - \sum_{(r,c)\in g}\beta_{r,c}|$ is monotone on
either side of the target, its sublevel sets are nested intervals, and a
single-cell change moves the unit sum by at most $255$ --- a
warmer/colder signal pointing at $511$. Nothing of the sort survives
$p$-adically. The sublevel sets
$\{t \in \mathbb{Z}_p : |t|_p \le p^{-k}\} = p^k\mathbb{Z}_p$ are clopen
subgroups of a totally disconnected space, and $|\cdot|_p$ is locally
constant on $\mathbb{Z}_p \setminus \{0\}$. Since the attainable residuals
satisfy $|r_g(\beta)| \le 1793$, each unit term $|r_g(\beta)|_p$ depends on
the unit sum only through its residue modulo $p^{K}$ with
$K = \lfloor \log_p 1793 \rfloor + 1$, and takes at most $K+1$ values. The
loss is therefore constant on ultrametric plateaus:
at $p=11$ a unit summing to $390$ (residual $121 = 11^2$) receives the
near-perfect term $11^{-2}$, while a unit summing to $510$ (residual $1$)
receives the worst possible term $1$.

\begin{lemma}[Plateaus, near-misses, and all-or-nothing progress]
\label{lem:pow2-landscape}
Let $p$ be an odd prime and $g$ a unit with residual $r = r_g(\beta)$.
Consider a local move that changes one entry of $g$ from $2^b$ to $2^a$
with $a \ne b \in \{0,\dots,8\}$ (for a two-cell swap, each unit containing
exactly one of the cells is of this form), so that the residual becomes
$r' = r - \varepsilon$ with $\varepsilon = 2^{a} - 2^{b}$.
\begin{enumerate}
\item[(i)] $|\varepsilon|_p = |2^{|a-b|}-1|_p$, which equals $1$ unless
$\mathrm{ord}_p(2)$ divides $|a-b| \le 8$. This divisibility is impossible
for $p \in \{11, 13, 23, 47, 73, 97\}$, whose orders are
$10, 12, 11, 23, 9, 48$, and automatically for $p > 255 \ge |\varepsilon|$.
For all such primes every move is a $p$-adic unit step; in particular, a
state differing from a valid completion in a single cell has every affected
unit term equal to $1$, and is scored identically to a unit whose sum is
arbitrary and coprime to $p$.
\item[(ii)] If $r \ne 0$, the move strictly decreases the unit term if and
only if $\varepsilon \equiv r \pmod{p^{\,v_p(r)+1}}$: improvement requires
the move to reproduce the residual's leading $p$-adic digit exactly.
\item[(iii)] If $|\varepsilon|_p > |r|_p$ --- in particular whenever
$\varepsilon$ is a unit and the unit sum is exactly ($r=0$) or approximately
($v_p(r) \ge 1$) satisfied --- then $|r'|_p = |\varepsilon|_p$: a single
generic move erases all accumulated valuation on that unit at once.
\end{enumerate}
\end{lemma}

\begin{proof}
(i) $2^{a}-2^{b} = 2^{\min(a,b)}\bigl(2^{|a-b|}-1\bigr)$ and $|2|_p = 1$
for odd $p$; $p \mid 2^{d}-1$ if and only if $\mathrm{ord}_p(2) \mid d$,
and the listed orders all exceed $8$, while $0 < |\varepsilon| \le 255 < p$
settles $p > 255$. A single-cell error leaves each affected unit with
residual $\pm\varepsilon$. (ii) $|r-\varepsilon|_p < |r|_p$ if and only if
$v_p(r-\varepsilon) > v_p(r)$, i.e.\ $p^{\,v_p(r)+1} \mid r - \varepsilon$.
(iii) The strong triangle inequality holds with equality when the two norms
differ.
\end{proof}

The problem can be exhibited exactly. For $p = 7$
($\mathrm{ord}_7(2) = 3$), note $511 = 7 \cdot 73$, so the target itself is
$7$-adically small and any unit whose sum is divisible by $7$ has
$|r_g|_7 \le 7^{-1}$. Concretely, the near-miss unit holding eight correct
entries and $2^{0}=1$ in place of $2^{8}=256$ has residual
$255 = 3\cdot 5\cdot 17$ and term $1$, while the scrambled multiset
$(8,4,2,2,1,1,1,1,1)$ sums to $21$, has residual $490 = 2\cdot 5\cdot 7^{2}$,
and receives the term $7^{-2}$: eight-ninths of a solved unit is ranked
forty-nine times worse than noise. The same problem occurs at $p = 73$.
For the remaining small sweep primes the occasional sub-unit steps reward
divisibility of the digit gap by $\mathrm{ord}_p(2)$ --- $p=7$ discounts
gaps $3$ and $6$, $p=17$ the single gap $8$ ($255 = 3\cdot 5\cdot 17$),
$p=127$ the gap $7$ ($127 = 2^{7}-1$) --- congruences between the exchanged
digits with no relation to correctness. For $p=2$ the valuation of $r_g$
counts how many low-order binary digits of the
unit sum agree with $511 = 111111111_2$, a carry condition on the multiset
of entries, again independent of how many cells hold their correct values.

Under the heuristic model that unit sums equidistribute modulo powers of $p$,
Lemma~\ref{lem:pow2-landscape}(ii) makes the per-unit improvement event a
leading-digit hit of density about $1/p$. Lemma~\ref{lem:pow2-landscape}(iii)
makes all partial progress fragile: a cell lies in three units
whose residuals shift in lockstep, so repairing one unit while protecting
another requires simultaneous congruences. For $p \ge 43$ one has
$p^{2} \ge 1849 > 1793$, so the attainable valuations are $0$ or $1$ and
each unit term is $0$, $p^{-1}$, or $1$; for $p > 1793$ --- the sweep's
$p = 2311$ --- every nonzero attainable residual is a unit and the loss
collapses exactly to the number of unsatisfied unit sums, an integer in
$\{0,\dots,27\}$ that a move improves only by making some unit sum exactly
$511$. At small $p$ the landscape is arithmetic noise; at large $p$ it is
an exact-hit count; at no prime does the term grade progress toward a
permutation. The archived traces are consistent with plateau termination
rather than descent into structured basins: the greedy runs ended after
$12.5$ and $19.3$ steps on average with mean lift at most $0.009$, and
annealing achieved mean lift $0.000$ over $5000$-step runs
(Table~\ref{tab:powers-two-results}).

The pairwise construction of Section~\ref{sec:sudoku} is immune precisely because it never
asks the norm to measure anything. Its off-target residuals on the domain
are differences of digits, bounded by $8 < p$, so
Lemma~\ref{lem:edge-indicator} collapses every
edge term to an equality indicator; the loss becomes an integer conflict
count whose single-swap increments are small, informative integers
(Lemma~\ref{lem:row-swap-locality}), and local search behaves as it does on
any Max-CSP objective.
The powers-of-two encoding fails this twice over: at moderate primes its
off-target residuals carry spurious valuations, and at $p > 1793$, where
the norm does degenerate to an indicator, it indicates only $27$ exact-sum
events, each flipped by a vanishing fraction of moves, where the pairwise
family fields $810$ local indicators, several of which change under every
swap. The compilation principle is that a $p$-adic residual should be
engineered so that its off-target values are units --- the norm employed as
an indicator, never as a measurement of size --- and that the indicators be
numerous and local enough for the neighbourhood structure to see them. The
construction of Sections~\ref{construction} and~\ref{sec:sudoku} satisfies
both by design. This is the concrete instance promised in
Section~\ref{sec:preliminaries}: the local constancy that starves the search
here is the same non-Archimedean geometry that makes signed weights safe
there.

\section*{Funding}
This research was supported by an Australian Government Research Training Program Scholarship.

\section*{Compliance with ethical standards}
This article does not contain any studies with human participants or animals performed by the author.

\section*{Conflict of interest}
The author declares that he has no conflict of interest.

\section*{Author contributions}
This is a single-author paper. The author is responsible for the conceptual development, proofs, implementation, and writing.

\section*{Code and data availability}
The Python implementation used for the heuristic experiments is publicly available at:
\begin{quote}\footnotesize
\url{https://github.com/solresol/sudoku-padic-regression}\\
Dataset archive: \href{https://huggingface.co/datasets/gregb/sudoku-padic-regression-experiments}{Hugging Face dataset}\\
Experiment-code revision: \texttt{ba49a8b2e51670f66db624d814d3f31678881b93}\\
Main scripts: \path{code/padic_sudoku_regression.py}, \path{code/run_experiments.py}\\
Mihara-inspired comparison script: \path{code/padic_comparison_algorithms.py}\\
Powers-of-two script: \path{archive/scripts/run_experiments.py}\\
Raw E1 data: \path{archive/results/e1_prime_sweep.csv}\\
Raw E2 data: \path{archive/results/e2_heuristic_comparison.csv}
\end{quote}
The submitted source package includes the experiment CSV files, puzzle seeds, trace puzzle and solution files, and figure source used for Tables~\ref{tab:results} and \ref{tab:powers-two-results} and Figure~\ref{fig:loss}. The raw-data revision above generated the archived CSV rows; the submitted driver corrects the summary formatter so that half-integer medians are not truncated. The main experiment outputs were produced with Python~3.11.6 and Matplotlib~3.9.0 on an Apple M1 system. Timing columns are retained as descriptive provenance only.

\subsection*{Interactive demonstration}
A browser-based demonstration of the construction is shown in Figure~\ref{fig:browser-sudoku-solution} and available at
\begin{quote}\small
\url{https://padic-logic.symmachus.org}\\
\url{https://padic-logic.symmachus.org/\#sudoku}
\end{quote}
The Sudoku and Boolean-CSP compilers, residual dataframes, objective evaluations, and representative searches run client-side. Their source is maintained separately at \url{https://github.com/solresol/padic-logic}, with the paper-linked application snapshot pinned at \href{https://github.com/solresol/padic-logic/commit/a8130ee7784035d17fc2aea3af56eda27bf478e1}{revision \texttt{a8130ee}}; the deployed values shown in Figure~\ref{fig:browser-sudoku-solution} were checked on 13 July 2026. The optional natural-language schema extractor requires a browser exposing the on-device Prompt API and is therefore not part of the browser-independent reproducibility claim.

\end{document}